\documentclass[sigconf]{acmart}
\acmBooktitle{ }
\usepackage{amsmath,amsthm,amsfonts}
\usepackage{bbm}
\usepackage{hhline}
\usepackage{multirow}
\usepackage{graphicx}
\usepackage{epstopdf}
\usepackage{array}
\usepackage[mathscr]{eucal}
\DeclareMathAlphabet\mathbfcal{OMS}{cmsy}{b}{n}
\usepackage{multirow}
\usepackage{colortbl}
\usepackage{tabulary}
\usepackage{etoolbox}
\usepackage[ruled,vlined,linesnumbered]{algorithm2e}

\usepackage[font={small}]{caption}
\usepackage[font={footnotesize}]{subcaption}
\usepackage{cancel}
\usepackage[toc,page]{appendix}
\usepackage[flushmargin]{footmisc}
\usepackage{pifont}

\let\oldnl\nl
\newcommand{\nonl}{\renewcommand{\nl}{\let\nl\oldnl}}
\renewcommand\footnotetextcopyrightpermission[1]{}
\begin{document}

\title[Suppressing Noise from Built Environment Datasets to Reduce Communication Rounds]{Suppressing Noise from Built Environment Datasets to Reduce Communication Rounds for Convergence of Federated Learning}
\author{Rahul Mishra, Hari Prabhat Gupta, Tanima Dutta, and *Sajal K. Das}
\affiliation{%
 \institution{Dept. of CSE, IIT (BHU) Varanasi, India and *Dept. of Computer Science, Missouri S\&T, Rolla, USA}
}
\email{{rahulmishra.rs.cse17, hariprabhat.cse, tanima.cse}@iitbhu.ac.in,  and  *sdas@mst.edu}

\begin{abstract}
Smart sensing provides an easier and convenient data-driven mechanism for monitoring and control in the built environment. Data generated in the built environment are privacy sensitive and limited. 
Federated learning is an emerging paradigm that provides privacy-preserving collaboration among multiple participants for model training without sharing private and limited data. The noisy labels in the datasets of the participants degrade the performance and increase the number of communication rounds for convergence of federated learning. Such large communication rounds require more time and energy to train the model. In this paper, we propose a federated learning approach to suppress the unequal distribution of the noisy labels in the dataset of each participant. The approach first estimates the noise ratio of the dataset for each participant and normalizes the noise ratio using the server dataset. The proposed approach can handle bias in the server dataset and minimizes its impact on the participants' dataset. Next, we calculate the optimal weighted contributions of the participants using the normalized noise ratio and influence of each participant. We further derive the expression to estimate the number of communication rounds required for the convergence of the proposed approach. Finally, experimental results demonstrate the effectiveness of the proposed approach over existing techniques in terms of the communication rounds and achieved performance in the built environment. 
\end{abstract}

\maketitle

\section{Introduction}
Built environment is a human-made environment that provides space to live, work, and recreate~\cite{builtenv}. Integrating the built environment with smart sensing generates datasets that can be exploited for effective monitoring and control. Such datasets are privacy sensitive as they possess extensive characterization of the built environment. They are also limited due to privacy concerns and a lack of smart sensing infrastructure in the built environment. 

Federated Learning (FL) collaboratively trains a shared model on the server using multiple participants while keeping local and limited data decentralized to preserve privacy~\cite{mcmahan2017}. Prior studies utilized FL in various built environment applications, including the (electrical) power load forecasting~\cite{gao2021decentralized} and anomaly detection~\cite{sater2021federated}. Noisy labels in the datasets of the participants are one of the key challenges that decelerate the broader acceptability of FL. They degrade the performance of the trained model~\cite{yang2020robust} and increase the communication rounds between the participants and the server for convergence of the FL~\cite{9412599}. The noisy labels also deteriorate the quality of the Weight Parameter Matrices (WPM) of the participants. Additionally, handling the unequal distribution of noisy labels among different classes becomes extremely challenging. 

There exists literature on handling noisy labels in participants' dataset to optimize the weighted contribution during aggregation~\cite{chen2020focus, li2021federated, xue2021toward}. However, training in the presence of noisy label instances in the dataset impacts the performance of the model. Prior studies removed the noisy label instances from the participants' dataset~\cite{9412599, yang2020robust, Han2020} to minimize the impact of noise in FL. The mechanism of sanitary checks to classify the labels as noisy or noise-free is proposed in~\cite{Han2020}, where only a few noise-free instances were used in the model training on each participant. Such few instances increase the communication rounds for convergence and degrade the performance. Exchanging the centroids between the participants and the server~\cite{9412599} to remove the noise reveals the characteristics of the dataset, which results in privacy compromise. Determining the noisy labels using a single model trained by a third party~\cite{yang2020robust} reduces the significance of the personalization. Further, there exist methods~\cite{chen19g,hendrycks2018, yu2018efficient} in centralized training to handle noisy labels; such methods do not perform well in FL due to the customary smaller sizes of the participant datasets and privacy concerns. 

The number of communication rounds and the duration of each round determine the convergence time of FL. The duration of each round depends on the local epochs and the number of participants. Therefore, the number of communication rounds help estimate the required number of local epochs and participants for achieving the desired accuracy of FL in fixed time duration. The authors in~\cite{li2019convergence,luo2021cost,yang2021achieving} estimated the number of communication rounds for the convergence of FL, assuming all the samples have a similar impact on the training. The assumption of a similar impact may not be valid in the presence of noisy labels. This is because the random selection of samples also has noisy labels, deteriorating the performance of the trained model. For example, FedProx~\cite{li2020federated} and FedNova~\cite{tackling} exhibit communication rounds and accuracy conflict, where faster convergence (i.e., fewer communication rounds) come at the expense of accuracy. While existing approaches~\cite{li2019convergence,luo2021cost,yang2021achieving, li2020federated, tackling} estimate the communication rounds for convergence of FL with noise-free labels in the participants' local dataset, estimating the communication rounds (i.e., the convergence time) in the presence of noisy labels with a given accuracy threshold is an important challenging problem in FL.

This paper proposes {\bf Fed-NL}, a novel \underline{Fed}erated learning approach to minimize the impact of unequal distribution of \underline{N}oisy \underline{L}abels in the participants dataset in built environments. Fed-NL estimates the noise ratio of the dataset for each participant and normalizes the unequal distribution of noisy labels, assuming the server has a noise-free dataset, as in existing works~\cite{9412599, chen2020focus, yang2020robust,9632275}. Each participant demands class-wise data samples from the server without revealing the characteristics of the local dataset. Fed-NL also estimates the communication rounds required for the convergence of FL in the presence of noisy labels.\\ Specifically, we investigate the following problem: \textit{how to minimize the impact of unequal distribution of noisy labels in the dataset of each participant while estimating the communication rounds for the convergence of FL?}

\vspace{2pt}
\noindent{\bf Contributions and novelty:} The main contributions and novelty of this work are as follows:

\vspace{2pt}
\noindent 1) The proposed approach Fed-NL estimates the noise ratio in the dataset of each participant without using any prior information on noise concentration. Each participant randomly splits its dataset into multiple disjoint sub-datasets; a model is trained on one of the sub-datasets and tested on the remaining ones. Such training and testing operations for all sub-datasets are repeated to obtain a dataset with noise-free labels, thus estimating the noise ratio. This estimation is independent of server, third parties, or other participants in FL.

\vspace{2pt}
\noindent 2) Fed-NL normalizes the estimated noise ratio to eliminate the unequal distribution of noisy labels. The class having the least noise ratio is a threshold for the remaining classes, which demand noise-free data samples from the server to achieve the threshold. The normalization reduces the noisy labels on each participant without revealing the characteristics of the local dataset.  

\vspace{2pt}
\noindent 3) Fed-NL determines the least number of samples the server transfers to a selected participant against a given class. Afterward, the participant fetches class-wise data samples from the server with the least sample count as the threshold. It neutralizes the impact of bias in the server dataset on the participant dataset.

\vspace{2pt}
\noindent 4) Fed-NL estimates the optimal weighted contribution of each participant by incorporating the normalized noise ratio and influence of participants~\cite{xue2021toward}. The optimal contribution provides robustness against the noise during aggregation. Expressions are derived to analyze the convergence of Fed-NL and the communication rounds using the following function:

\begin{equation}\label{e0}
 \mathcal{R}= f\Big(E,N,q_o,\epsilon_i,\mathbb{E}[\nabla\mathcal{L}_i(\cdot,\cdot)]\Big)
\end{equation}

\noindent
where $E$, $N$, $q_o$, and $\epsilon_i$ respectively denote the number of local epochs, the number of participants, the precision threshold, and the contribution of each participant $i \in \{1\le i \le N\}$. Here $\mathbb{E}[\nabla\mathcal{L}_i(\cdot,\cdot)]$ is the expected value of loss function gradient using randomly chosen data samples in the presence of noisy labels. (For details on the expression for communication rounds $\mathcal{R}$, see Theorem~\ref{theorem}.)  

\vspace{2pt}
\noindent 5) The performance of Fed-NL approach is evaluated on our collected and existing datasets in transportation~\cite{shl2} and human activity recognition~\cite{anguita2013public}. Experimental results show that Fed-NL provides robustness against noise and successfully estimates the communication rounds in the built environments. Fed-NL is also evaluated on existing image datasets~\cite{caldas2018leaf, krizhevsky2009learning} used in federated learning.

\vspace{2pt}
The paper is organized as follows.
Section~\ref{related-work} reviews the related work and Section~\ref{problem} summarizes preliminary concepts. Sections~\ref{flapproach} and ~\ref{round} respectively describe the Fed-NL approach for handling noisy labels in the datasets, and estimates the communication rounds required for convergence of Fed-NL. Section~\ref{evaluation} evaluates the performance of Fed-NL and Section~\ref{conc} concludes the paper.

\section{Related work}\label{related-work}

This section discusses the existing works on how to handle noisy labels in the participant datasets and estimate the communication rounds for convergence of FL. 

\subsection{Handling noisy labels in FL}
Existing literature on handling noisy labels typically uses the following mechanisms: a) optimizing weighted contributions, b) reducing noisy labels, and c) learning with noisy labels. 

The weighted contribution of the participants during aggregation at the server can be optimized to mitigate the negative impact of noisy labels in FL~\cite{chen2020focus, li2021federated, xue2021toward}. In~\cite{chen2020focus}, the authors highlighted the issue of noisy labels in the dataset and computed mutual cross-entropy loss among the participants and the server to assign different weight contributions. However, the technique could not handle bias due to the unequal distribution of noisy labels in the dataset. In~\cite{li2021federated}, the authors proposed a framework to incorporate a data quality measurement mechanism, thereby providing robust aggregation against noisy labels in the training dataset. In the proposed framework, the noisy data instances remained confined to the participants, which hampered the performance of the trained model. The influence of participants in FL is estimated in~\cite{xue2021toward} to quantify the impact on the weighted contribution. The authors removed less influential participants from FL, where eliminated participants could not improve their performance. 

The negative impact of noisy labels in the dataset can be reduced by identifying and removing the instances with such labels~\cite{9412599, yang2020robust, Han2020}. The authors in~\cite{9412599} introduced a technique to select relevant data for model training on the participants. However, exchanging the centroids between the participants and the server~\cite{9412599} to remove the noisy labels reveals the characteristics of the dataset, which results in privacy compromise. The authors in~\cite{yang2020robust} highlighted the incorrect annotation of participants' datasets and proposed a method for choosing consistent decision boundaries to identify noisy and noise-free samples. Determining consistent decision boundary is effective, but identifying the noisy labels using a single unified model trained on a third party reduced the significance of the personalization in FL. A mechanism of sanitary checks is introduced in~\cite{Han2020} to classify the class labels as noisy or noise-free. The authors utilized only a few noise-free instances for training the model using the student-teacher training approach. However, removing noisy labels from the dataset and working with limited samples increases the local epochs for training and leads to performance compromise of FL~\cite{yang2020robust}. Additionally, model training with the unequal noise distribution among the participants and server (\textit{i.e.,} students and teacher) degrade the performance of FL.

In prior studies, several methods have been designed to provide effective learning despite noisy labels in the dataset~\cite{chen19g,hendrycks2018}. In~\cite{hendrycks2018}, a set of trusted data with clean labels is used to resolve the negative impact of noisy labels. A classifier is initially built using limited and noise-free data, followed by training on all the instances. However, considering noisy data after initial training on noise free data hampered the performance of the trained model in FL. The classification problem with training instances having randomly corrupted labels is addressed in~\cite{7159100}, where the authors introduced the importance of data instances re-weighting technique to best utilize the loss functions designed for traditional classification problems. Similar to~\cite{hendrycks2018}, the training in FL used noisy labeled data instances, which degraded the performance of the trained model.  

\subsection{Communication rounds in FL}  
The existing work have estimated the number of communication rounds required for convergence of FL~\cite{li2019convergence, luo2021cost, yang2021achieving}. As regards to the convergence of FL, the authors in~\cite{li2019convergence} estimated the number of communication rounds in FL, considering participants' datasets with noise-free labels. Similarly, the authors in~\cite{luo2021cost} analyzed the role of three control variables in FL, \textit{i.e.}, the communication rounds, local epochs, and the number of participants to determine the number of communication rounds for convergence of FL. The derived expression for communication round established a reciprocating relation between communication rounds and local epochs, \textit{i.e.} when local epochs increased, the number of communication round decreased. The authors in~\cite{yang2021achieving} extended the estimation of communication rounds in~\cite{li2019convergence} for iid and non-iid datasets with partial and complete involvement of all the participants. They proved that the linear speedup is achieved under both iid and non-iid assumptions. 

The estimation of communication rounds for convergence of FL in~\cite{li2019convergence,luo2021cost,yang2021achieving} is applicable for participants' datasets with noise free labels, where the assumption for a similar impact of all the samples in the dataset is valid. However, such an assumption of a similar impact may not be applicable for participants' datasets with noisy labels. Therefore, estimating communication rounds with noisy labels is tedious and requires redefining the variables in expression for communication rounds of FL with noise-free labels.

\section{Preliminaries}\label{problem}
This work considers an FL scenario for multi-class classification problem with $c$ classes of set $[c]=\{1,2,\cdots c\}$. A set of $N$ participants in FL is denoted by $[p]=\{p_1, p_2, \cdots, p_N\}$. Each participant $p_i$ has dataset $\mathcal{D}_i$ with $n_i$ number of instances and classes $[c]$, where $1\le i \le N$. Similarly, let $\mathcal{D}^s$ denotes the server dataset with classes $[c]$. Let $\mathcal{D}_{ik}$ and $\mathcal{D}^s_k$ denote the sub-datasets of $\mathcal{D}_{i}$ and $\mathcal{D}^s$ for class $k$, respectively, $\forall k\in[c]$. This work assumes that $\mathcal{D}_i$ may have noisy labels, $\forall i \in (1\le i \le N)$. Let $(\mathbf{x}_{ij}, y_{ij})$ denotes an instance of dataset $\mathcal{D}_i$, where $1\le j \le n_i$. Let $\bar{y}_{ij}$ and $y^o_{ij}$ denote the true and predicted labels by the model of an instance $\mathbf{x}_{ij}$, respectively, where $\bar{y}_{ij}, y^o_{ij} \in [c]$. An instance $\mathbf{x}_{ij} \in \mathcal{D}_i$ has noise-free label, if the observed $y_{ij}$ matches with $\bar{y}_{ij}$; however, $\bar{y}_{ij}$ is usually unavailable. Let $\mathbf{x}_{ij}$ denotes an instance of dataset $\mathcal{D}_i$ at participant $p_i$ having class label $y_{ij}$, where $1\le i \le N$, $1\le j \le n_i$, and $y_{ij}\in [c]$. The instance $\mathbf{x}_{ij}$ is said to have noisy label if either of the following conditions hold: a) $y_{ij}$ is mislabeled as any other class of $\mathcal{D}_i$, \textit{i.e.}, $y_{ij} \in [c]-\bar{y}_{ij}$, where $\bar{y}_{ij}$ denotes the true class label. b) $y_{ij}$ is an arbitrary class label, \textit{i.e.,} $y_{ij} \notin [c]$. This work assumes that the server comprises a dataset with noise-free labels. The existing literature in FL inspires this assumption~\cite{9412599, chen2020focus, yang2020robust, 9632275}, where the server is assumed to possess a dataset. Moreover, the assumption does not reveal the characteristics of participant datasets; thus, it maintains data privacy in FL.

\noindent $\bullet$ \textbf{Federated learning:} Each participant $p_i$ in FL receives a model from the server and trains it using dataset $\mathcal{D}_i$. The training is performed for $E$ number of local epochs and minimizes the loss function $\mathcal{L}_i(\mathbf{w}_i,\mathcal{D}_i)$ of participant $p_i$, where $\mathbf{w}_i$ is the WPM of $p_i$. $\mathcal{L}_i(\mathbf{w}_i,\mathcal{D}_i)$ is estimated as: $\mathcal{L}_i(\mathbf{w}_i,\mathcal{D}_i)=\frac{1}{n_i}\sum_{j\leftarrow 1}^{n_i} \mathcal{L}_{ij}(\mathbf{w}_{ij},\mathbf{x}_{ij})$, where  $\mathbf{x}_{ij}\in \mathcal{D}_i$, $1\le j \le n_i$, and $1\le i \le N$. Participant $p_i$ estimates $\mathcal{L}_i(\cdot,\cdot)$ and transfers $\mathcal{L}_i(\cdot,\cdot)$ and $\mathbf{w}_i$ to the server for global aggregation. Next, the server estimates global loss function as $\mathcal{L}(\mathbf{w},\mathcal{D}^s)=\sum_{i\leftarrow 1}^{N} \Big(\frac{n_i}{\mathcal{N}}\Big)\mathcal{L}_i(\mathbf{w}_i,\mathcal{D}_i)$ and WPM $\mathbf{w}=\sum_{i\leftarrow 1}^{N} \Big(\frac{n_i}{\mathcal{N}}\Big)\mathbf{w}_i$ using weighted aggregation of the local losses and WPM of participants, where $\mathcal{N}=n_1+n_2+\cdots+n_N$. 

The server broadcasts the aggregated WPM to all the participants for the next round of training. The process of local training and global aggregation is orchestrated for $\mathcal{R}$ global iterations to achieve a trained model for all the participants. At each global iteration $t\in \mathcal{R}$ the local loss function and WPM are denoted as $\mathcal{L}_i^t(\cdot,\cdot)$ and $\mathbf{w}_i^t$ for participant $p_i$, respectively, where $1\le i \le N$. Similarly, the server has aggregated loss function and WPM as $\mathcal{L}^t(\cdot,\cdot)$ and $\mathbf{w}^t$ at iteration $t$, respectively. During local training for $E$ epochs the WPM $\mathbf{w}_i^t$ at $t$ ($t\in \mathcal{R}$) of participant $p_i$ is updated as: $\mathbf{w}_{i,k}^t = \mathbf{w}_{i,e-1}^t-\eta \nabla\mathcal{L}_i(\mathbf{w}_{i,e-1}^t, \mathcal{D}_i), \forall e \in E$, where $\eta$ is the learning rate and $\mathbf{w}_{i,0}^t$ is initialized as: $\mathbf{w}_{i,0}^t\leftarrow \mathbf{w}^{t-1}$. 

\noindent $\bullet$ \textbf{Influence of participants:} 
The influence of $p_i$ on aggregated WPM $\mathbf{w}^t$ at iteration $t$ ($t\in \mathcal{R}$), denoted as $\delta_i^t$, can be estimated by removing the WPM $\mathbf{w}_i^t$ from aggregation, where $1\le i \le N$. Let $[p^t]$ of $N$ denotes the set of participants available at iteration $t$. Let $\mathcal{N}([p^t]\textbackslash {p_i})$ denotes the sum of instances of all the participants excluding $p_i$ then aggregated WPM at iteration $t$ is given as~\cite{xue2021toward}:
\begin{align}\label{e4}
 \mathbf{w}^t([p^t] \textbackslash {p_i})&=\sum_{l\in \{[p^t]\textbackslash {p_i}\}}\Big(\frac{n_l}{\mathcal{N}([p^t]\textbackslash {p_i})}\Big)w^t_l.
 \end{align}

The influence $\delta_i^t$ of $p_i$ when it is permanently removed from the training is given as:
\begin{equation}\label{e5}
\delta_i^t= \mathbf{w}^t([p] \textbackslash {p_i}) - \mathbf{w}^t. 
\end{equation}

\noindent $\bullet$ \textbf{Distribution of noisy labels:}
There are two types of distribution of noisy labels in the dataset of the participants~\cite{chen19g}, \textit{i.e.}, symmetric and asymmetric. The data instances in the symmetric distribution of the noisy labels having the same distance to the class boundary are equally likely to be corrupted. The noise ratio of participant $p_i$, denoted as $\beta_i$, is the fraction of the noisy instances to the total instances of $p_i$. The noise transition matrix $P$ for symmetric distribution of the noisy labels is defined as: $P_{kk}=1-\beta_i$ and $P_{kl}=\beta_i/(c-1)$, where $k,l\in[c]$. 
In the asymmetric distribution, the data instances with the same class boundary distance are not equally likely to be corrupted. $P$ for the asymmetric distribution of noisy labels is given as: $\sum_{k,l \in [c]}P_{kl}=1$, where $P_{kk}>P_{kl}$, others $P_{kl}$ may be unequal, and  $\forall k\ne l$. 

\section{Fed-NL: \underline{Fed}erated Learning Approach for handling \underline{N}oisy \underline{L}abels}\label{flapproach}
This section presents the Federated Learning approach with Noisy Labels (Fed-NL). An overview of the Fed-NL is shown in Figure~\ref{overview1}. Fed-NL starts with the estimation of the noise ratio on each selected participant in FL. Next, we introduce the noise ratio normalization mechanism because the obtained noise ratio of a participant may not be equally distributed among all the class labels. We also determine the least data samples the server can transfer to each participant class-wise. Further, we estimate the optimal contribution of each participant as per the noise ratio and influence. We summarize these steps in Algorithm 1. Finally, we derive the expression to estimate the communication rounds $\mathcal{R}$ for convergence of Fed-NL.  

\begin{figure}[h]
\centering
\includegraphics[scale=0.76]{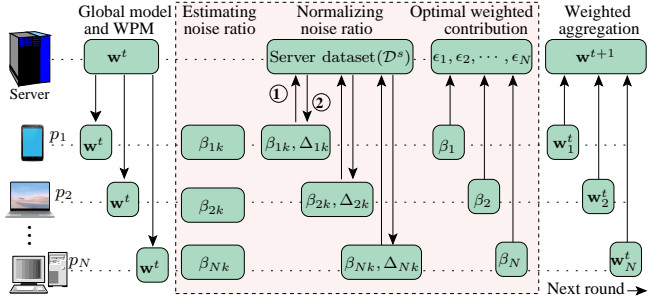}
\caption{Overview of Fed-NL. \textcircled{1}: data request to the server and \textcircled{2}: data provided by the server.} 
\label{overview1}
\end{figure}

\subsection{Estimating noise ratio at each participant}\label{noise-ratio}
Each participant in Fed-NL randomly splits its dataset into multiple disjoint sub-datasets; a model is trained on the one sub-dataset and tested on the remaining. The authors in~\cite{chen19g} split the dataset into two sub-datasets to determine noisy labels by alternatively training and testing on the first and second sub-dataset. It provided one predicted and one existing label to identify noise-free labels. However, it would not work if the existing and the predicted labels are noisy; the label is marked as noise-free. To overcome the above limitation, the estimation of noise ratio in Fed-NL starts with the random split of $\mathcal{D}_i$ of participant $p_i$ into three sub-datasets having a nearly equal instances, denoted as $\mathcal{D}_i^1$, $\mathcal{D}_i^2$, and $\mathcal{D}_i^3$, where $1\le i \le N$.

A participant $p_i$ initially trains the local model $M_i$ on $\mathcal{D}_i^1$ and the data instances of $\mathcal{D}_{i}^2$ and $\mathcal{D}_{i}^3$ are tested using the trained $M_i$, $\forall i \in \{1\le i \le N\}$. The testing operation results in a set of output labels for $\mathcal{D}_{i}^2$ and $\mathcal{D}_{i}^3$, respectively. $p_i$ next re-initializes and trains $M_i$ on $\mathcal{D}_{i}^2$ and tests $\mathcal{D}_{i}^1$ and $\mathcal{D}_{i}^3$ instances on trained $M_i$ to obtain set of output labels. $p_i$ finally re-initializes and trains $M_i$ on $\mathcal{D}_{i}^3$ and tests $\mathcal{D}_{i}^1$ and $\mathcal{D}_{i}^2$ to obtain set of output labels. Now, we have three labels for each instance of any sub-dataset, \textit{i.e.}, one existing and two predicted. An instance of a sub-dataset is noise-free if its existing label is the same as two predicted labels, as shown in Example~\ref{example1}. Fed-NL uses this mechanism to estimate class-wise noise-free labels in each sub-dataset and assign to a set $S_{ik}^q$, where $i\le i\le N$, $k\in[c]$, and $q\in\{1,2,3\}$. For example, all the instances with noise-free labels in $\mathcal{D}_{ik}^1$ are assigned to $S_{ik}^1$. Similarly, we determine $S_{ik}^2$ and $S_{ik}^3$. Further, we obtain a noise-free set $S_{ik}$ by taking the union of sets $S_{ik}^1$, $S_{ik}^2$, and $S_{ik}^3$, \textit{i.e.,} $S_{ik}=S_{ik}^1\cup S_{ik}^2\cup S_{ik}^3$. The removed set $R_{ik}$ is estimated by following expression: $R_{ik}= \mathcal{D}_{ik}-S_{ik}$. 

\begin{example}\label{example1}
Let us consider a binary classification task, where a model is trained on a noisy dataset of any participant. An instance of the dataset, can have the following possibilities of labels using the three-fold split of the dataset.
\begin{align}\label{exeq}
 \begin{matrix}
 \\
\text{Noisy}\\ 
\text{Noisy}\\ 
\text{Noisy}\\ 
\text{Noise-free} 
\end{matrix}\begin{bmatrix}
\text{Existing} & \text{Predicted 1} & \text{Predicted 2}\\
0 & 1 & 1 \\ 
0 & 1 & 0 \\ 
0 & 0 & 1 \\ 
\textbf{0} & \textbf{0} & \textbf{0} 
\end{bmatrix}.
\end{align}
We set the last row of \eqref{exeq} as noise-free as the existing label and predicted labels (Predicted 1 and Predicted 2) from alternative training and testing are same. In other rows, there is conflict among existing and predicted labels. The conflict is due to noisy labels or improper training owing to the insufficiency of the dataset. 
\end{example}
The noise ratio of participant $p_i$ for class $k$ is estimated as: $\beta_{ik}=\frac{|R_{ik}|}{|\mathcal{D}_{ik}|}$, where $|\cdot|$ is the cardinality. The participant $p_i$ estimates the noise ratio for all the classes. Similarly, all the participants estimate the class-wise noise ratio. Procedure~$1$ summarizes the steps involved in estimating the noise ratio in Fed-NL. 


\SetAlFnt{\small}
\begin{procedure}[t]
\label{procedure1}
\caption{() \textbf{1: Estimation of Noise Ratio at $p_i$.}}  
\KwIn{Dataset $\mathcal{D}_i$, local training epoch $E$\;}
\For{each class $k \in [c]$}{
Randomly split $\mathcal{D}_i$ in three sub-datasets ($\mathcal{D}_i^1$, $\mathcal{D}_i^2$, and $\mathcal{D}_i^3$)\;
\For{$j\leftarrow 1$ to $3$}{
$a \leftarrow j, b\leftarrow j+1,c \leftarrow j+2$\;
\If{$c\%4==0$}{
$c\leftarrow a-1$\;
}
\ElseIf{$b\%4==0$}{
$b\leftarrow a-2$
}
Select $\mathcal{D}_{ik}^b \in \mathcal{D}^b_i$ and $\mathcal{D}_{ik}^c \in \mathcal{D}^c_i$ \;
Train $M_i$ on $\mathcal{D}^a_i$ for $E$ epochs; /*$M_{i}$ is local model at $p_i$*/\\
$S_{ik}^a\leftarrow$\textit{Test\_fun} ($\mathcal{D}_{ik}^b$, $\mathcal{D}_{ik}^c$, $M_i$)\;
\smallskip
Re-initialize the trained model $M_i$\;
}
\smallskip
$S_{ik}\leftarrow S^1_{ik} \cup S^2_{ik} \cup S^3_{ik}$; /*Noise-free set*/\\
Set $R_{ik}\leftarrow \mathcal{D}_{ik}-S_{ik}$ and $\beta_{ik}\leftarrow \frac{|R_{ik}|}{|\mathcal{D}_{ik}|}$\; 
}
\Return $S_{ik}$, $R_{ik}$ and $\beta_{ik}$, where $1\le i \le N$ and $k\in[c]$\;

\nonl \textbf{Function} \textit{Test\_fun} (dataset $X_1$, dataset $X_2$, model $Y$)\\
\nonl \hspace{10pt}\textbf{begin} \nonumber \\
\nonl \hspace{20pt} \textbf{for} {each instance $\mathbf{x}_1 \in X_1$, $\mathbf{x}_2 \in X_2$, and $\mathbf{x}_1=\mathbf{x}_2$} \textbf{do} \\
\nonl \hspace{30pt} /*$y$ is the existing label of $\mathbf{x}_1$ or $\mathbf{x}_2$*/\\
\nonl \hspace{30pt} $Y$ predicts $y_1$ as label of $\mathbf{x}_1$, $Y$ predicts $y_2$ as label of $\mathbf{x}_2$\;
\nonl \hspace{30pt} \textbf{if} $y_1 = y$ \&\& $y_2 =y $ \textbf{then} \\
\nonl \hspace{40pt} Set $S\leftarrow append(\text{selected instance})$\;
\nonl \hspace{20pt}\Return $S$ \;
\nonl \hspace{10pt}\textbf{end}
\end{procedure}

\subsection{Class-wise normalization of noise ratio and data samples from server to participant}\label{normalize}
A class with minimum noise ratio $(\beta_{ik}=z_i)$ does-not demand data from the server, and its noise ratio acts as a threshold for other classes. The participant $p_i$ for other classes which have noise ratio more than $z_{i}$, demands a fraction $\mathcal{F}_{ik} \leftarrow (\beta_{ik}-z_i)$ of dataset $\mathcal{D}_{ik}$ from the server. 
The server uniformly at random selects $\mathcal{F}_{ik}.\mathcal{D}_{ik}$ portion of data instances from $\mathcal{D}_{k}^s$ and sends to the participant. However, the server dataset might be biased where data instances for some classes are low and high for others. Specifically, the server may be incompetent in fulfilling the demand of $|\mathcal{F}_{ik}.\mathcal{D}_{ik}|$ data instances, for a few classes $\in[c]$, from the participant $p_i$. It results in bias of the class instances on the participant. Therefore, we identify the least possible data instances the server could transfer to the participant among all classes, denoted as $\Delta_{ik}$, as shown in lines $8-17$ of Procedure~$2$. Further, the server transfers class-wise and uniformly at random selected $\Delta_{ik}$ data instances to the participant $p_i$ for classes having noise ratio $\ge z_{i}$.

After receiving $\Delta_{ik}$ data instances from the server, participant $p_i$ has the dataset $\widehat{S_{ik}}$ of class $k$, where $\widehat{S_{ik}} \leftarrow S_{ik}+\Delta_{ik}$. We impose a limit of $z_i$, on the fraction of the dataset demand to restrict a vast amount of data transmission from the server to the participants. It also helps in preserving some data instances local to the participants to retain the personalized features. The number of instances in $\widehat{S_{ik}}$ does not exceeds $\mathcal{D}_{ik}$. Thus, we do not require additional epochs for training the model on the newly obtained dataset at the participant $p_i$ for class $k$. Procedure~$2$ summarizes the steps involved in normalizing noise ratio and dataset from the server on a selected participant $p_i$ with dataset $\mathcal{D}_i$. 


\SetAlFnt{\small}
\begin{procedure}[t]
\label{procedure2}
\caption{() \textbf{2: Normalizing Noise Ratio of $p_i$ and Data Samples from the Server.}}  
\KwIn{Server Dataset $\mathcal{D}^s$ and participant dataset $\mathcal{D}_i$\;}
\For{each class $k \in [c]$ }{
Determine dataset $\mathcal{D}_{ik}$ from $\mathcal{D}_i$, where $\mathcal{D}_{ik}\subset \mathcal{D}_i$\;
 $S_{ik},R_{ik}, \beta_{ik}\leftarrow$ Call \textbf{Procedure~1}\; 
\If{k=1}{
$z_i \leftarrow \beta_{ik}$, $b_i\leftarrow k$\;
}
\ElseIf{$\beta_{ik}< z_i$}{
$z_i\leftarrow \beta_{ik}$, $b_i\leftarrow k$\;
}
}
\For{class $k \in \{1,2,\cdots c\}/\{\text{class with } \beta_{ik}==z_i\}$ in $\mathcal{D}_i$}{
Estimate fraction $\mathcal{F}_{ik} \leftarrow (\beta_{ik}-z_i)$\;
$p_i$ demand $\mathcal{F}_{ik}.\mathcal{D}_{ik}$ instances from the server\;
\If{$|\mathcal{D}_{k}^s|\le |\mathcal{F}_{ik}.\mathcal{D}_{ik}|$}{
$\Delta_{ik}^1 \leftarrow |\mathcal{D}_{k}^s|$
}
\Else{
$\Delta_{ik}\leftarrow |\mathcal{F}_{ik}.\mathcal{D}_{ik}|$;
}
}
$u\leftarrow \min \{\Delta_{i1}^1,\Delta_{i2}^1,\cdots,\Delta_{ic}^1\}$\; 
\If{$u \ne 0$}{
$\Delta_{i1}=\Delta_{i2}=\cdots=\Delta_{ic}=u$\;
}
\For{class $k \in \{1,2,\cdots c\}/\{\text{class with } \beta_{ik}==z_i\}$ in $\mathcal{D}_i$}
{
Server uniformly at random determine $\Delta_{ik}$ instances of $\mathcal{D}_{k}^s$\; 
Server send the selected fraction to $p_i$\;
$\widehat{S_{ik}} \leftarrow S_{ik}+\Delta_{ik}$ then $\mathcal{D}_{ik}\leftarrow \widehat{S_{ik}}$\;
Call \textbf{Procedure~1} to estimate $\beta_{ik}$\;
}
$\beta_i=\frac{1}{c}\sum_{k\leftarrow 1}^{c} \beta_{ik}$; /*Updated noise ratio $\beta_i$*/\\
\Return $\beta_i$\;
\end{procedure}

\subsection{Optimal weighted contribution}\label{algo}
We consider the test dataset of server, denoted as  $\mathcal{D}^s_{te}$, to evaluate influence $\delta_i^t$ (given in (\ref{e5})) of participant $p_i$ at global iteration $t$ ($t\in \mathcal{R}$). The exact estimation of $\delta_i^t$ is only possible when we remove the WPM of participant $p_i$ from the entire training process. After that, the retraining of the shared model is performed using the remaining participants. However, retraining for each participant at each global iteration is prohibitively inefficient. Therefore, we use the expression for the estimated value of $\delta_i^t$, given in~\cite{xue2021toward}. Let $\bar{\delta_i^t}$ denote the estimated value of $\delta_i^t$, which is denoted as:
\begin{align}\label{e6}
 \bar{\delta_i^t} = Q^t_i \delta_i^{t-1} + \mathbf{w}^t([p^t]\textbackslash p_i) - \mathbf{w}^t,
\end{align}
where $Q^t_i  =\sum_{l \in \{[p^t]\textbackslash p_i\}} \Big(\frac{n_l}{\mathcal{N}([p^t]\textbackslash p_i)}\Big)\prod_{e=0}^{E-1}(\mathbf{I}-\eta\mathbf{H}_{le}^{t})$ and $\mathbf{H}_{le}^{t}$ is the Hessian matrix of $l^{th}$ participant at local epoch $e$ and global iteration $t$ ($t\in \mathcal{R}$). The authors in~\cite{xue2021toward} updated the basic estimator $\bar{\delta_i^t}$. The update provides robustness against non-convex loss functions and reduces the computation for the Hessian matrix. This work adopts a similar strategy to update the basic estimator. 

This work considers datasets of the participants with noisy labels; therefore, we use effective noise-free data instances to estimate $Q^t_i$. We replace $n_l$ with $n_l \times \beta_l$ and  $\mathcal{N}([p^t]\textbackslash p_i)$ with $\mathcal{M}([p^t]\textbackslash p_i)$ to obtain $(Q^t_i){'}$, where $\mathcal{M}([p^t]\textbackslash p_i)=(n_1\beta_1 + n_2\beta_2+\cdots+n_N\beta_{N})-n_i\beta_i$. Using $(Q^t_i){'}$ and including noise ratio in the estimation of $\mathbf{w}^t([p^t]\textbackslash p_i)$ and $\mathbf{w}^t$, we can obtained modified influence of participant $p_i$, denoted as $\gamma_i^t$. Later, we use $\gamma_i^t$ to estimate the contribution $\epsilon_i^t$ of participant $p_i$ at global iteration $t$ ($t\in \mathcal{R}$). $\epsilon_i^t$ is estimated using following expression: $\epsilon_i^t= 1/\gamma_i^t(\Omega^t)$, where $\Omega^t=1/\gamma_1^t+1/\gamma_2^t+\cdots+1/\gamma_N^t$. We use $\epsilon_i^t$ to estimate aggregated WPM $\mathbf{w}^t$ on the server at global iteration $t$ ($t\in \mathcal{R}$), which is given as:
$\mathbf{w}^t=\sum_{i\leftarrow 1}^{N} \epsilon_i^t \mathbf{w}_i^t$. Procedure~$3$ summarizes the steps of estimating the weighted contribution of all the participants at the server.

\SetAlFnt{\small}
\begin{procedure}[t]
\label{procedure3}
\caption{() \textbf{3: Weighted Contribution at Server}} 
\KwIn{Global iteration $t$, noise ratio $\beta_i$, number of instances $n_i$, server's WPM $\mathbf{w}^t$, and participant's WPM $\mathbf{w}_i^t$;}
\For{$p_i\in\{p_1,p_2,\cdots,p_N\}$}{
Determine effective noise-free sample on $p_i$ using $\beta_i$ and $n_i$\;
Estimate $(Q^t_i){'}$ using $\beta_i \times n_i$\;
Calculate $\gamma_i^t$ using $(Q^t_i){'}$ and $\beta_i \times n_i$\;
}
Determine weighted contribution $\epsilon_i^t$\; 

\Return weighted contribution $\epsilon_i^t$ of participant $p_i$\;
\end{procedure}

%
Finally, the Fed-NL algorithm (Algorithm~\ref{algorithm1}) uses Procedure 1, Procedure 2, and Procedure 3 to handle the noisy labels in the datasets of participants. 

\SetAlFnt{\small}
\begin{algorithm}[t]
\caption{\textbf{Fed-NL Algorithm}}
\label{algorithm1}
\KwIn{Set $[p]$ of $N$ participants with datasets\;}
\KwOut{Trained model on $p_i$ $(1\le i \le N)$\;}
Server builds a model using random weights\; 
Server broadcasts the WPM to all the participants\;
\For{each participant $p_i \in \{p_1,p_2,\cdots, p_N\}$}{
Call \textbf{Procedure~1} to estimate noise ratio\;
Call \textbf{Procedure~2} to normalize noise ratio \& data from server\;
Call \textbf{Procedure~3} to estimate weighted contribution $\epsilon_i^t$ for each participant $p_i$ at $t$  \;
}
\While{$1 \le t\le \mathcal{R}$ /*global iterations*/}{
\For{each participant $p_i \in \{p_1,p_2,\cdots, p_N\}$}{
Calculate $\mathbf{w}^{t+1}_i$ send to the server\;   
}

$\mathbf{w}^{t+1} =\sum_{i \leftarrow 1}^{N} \epsilon_{i}^t \mathbf{w}^{t+1}_{i}$\;
Broadcast $\mathbf{w}^{t+1}$ to all the participants\;
$t\leftarrow t+1$\;
}
\textbf{return} Trained model on each participant\;
\end{algorithm}

\section{Convergence analysis and communication rounds of Fed-{NL}}\label{round}
This section presents the convergence analysis and derives the expression for the communication rounds required for convergence of Fed-NL. Apart from the convergence analysis of FedAvg in~\cite{li2019convergence} on a noise-free dataset, the noisy labels in the participant's dataset make a significant variance in the local loss function $\mathcal{L}_i(\cdot)$ of each participant and aggregated global loss function $\mathcal{L}(\cdot)$.

Let $\phi_i^t$ denotes the uniformly and randomly selected sample from the dataset of participant $p_i$ at global iteration $t$, where $1\le t \le \mathcal{R}$. Let $\nabla\mathcal{L}_i(\phi_i^t, \mathbf{w}_i^{t})$ and $\nabla\mathcal{L}_i(\mathbf{w}_i^{t})$ denote the gradients of loss function $\mathcal{L}_i(\cdot)$ on $\phi_i^t$ samples and entire samples of the dataset, respectively. Assumptions~$1$ and $2$ for loss functions are $L$-smooth and $\mu$-strongly convex, respectively, in~\cite{li2019convergence} are generalized; thus, we directly employed them in this work. Assumptions~$3$ and~$4$ in~\cite{li2019convergence} are not applicable for the participant with noisy labels as $\|\nabla\mathcal{L}_i(\phi_i^t, \mathbf{w}_i^{t})-\nabla\mathcal{L}_i(\mathbf{w}_i^{t})\|$ may not be same for different set of samples drawn randomly. Thus, we use expected value of gradient of loss function $\mathbb{E}[\nabla\mathcal{L}_i(\phi_i^t, \mathbf{w}_i^{t})]$ for samples randomly drawn from the dataset at $t$. We redefine Assumptions~$3$ and $4$ of~\cite{li2019convergence} as follows:\\
\noindent $\bullet$ Variance of gradients on each participant $p_i$ for noisy labels is bounded as: $\mathbb{E}\|\mathbb{E}[\nabla\mathcal{L}_i(\phi_i^t, \mathbf{w}_i^{t})]-\nabla\mathcal{L}_i(\mathbf{w}_i^{t})\|^2$ $\leq \sigma_i^2$. \\
\noindent $\bullet$ $\mathbb{E}\|\mathbb{E}[\nabla\mathcal{L}_i(\phi_i^t, \mathbf{w}_i^{t})]\|^2$ $\leq G^2$. Using property of expected value, \textit{i.e.,} $\mathbb{E}[\mathbb{E}(X)^2]=\mathbb{E}(X)^2$, we obtain $\|\mathbb{E}[\nabla\mathcal{L}_i(\phi_i^t, \mathbf{w}_i^{t})]\|^2 \leq G^2$.

Using the above assumptions, we obtain a relation between desired precision $q_o$, local epoch $E$, and global iterations $\mathcal{R}$, as given in Theorem~\ref{theorem}. The desired precision is defined as: $q_o=\mathbb{E}[\mathcal{L}(\mathbf{w}^\mathcal{R})]-\mathcal{L}^{*}$, where $\mathbf{w}^\mathcal{R}$ is the aggregated weight at final global epoch $\mathcal{R}$ and $\mathcal{L}^{*}$ is minimum and unknown  value of $\mathcal{L}$ at the server. Let $\mathcal{L}^{*}_i$ is the minimum value of $\mathcal{L}_i$ at $p_i$, where $\forall i \in \{1\le i \le N\}$. This work assumes an iid dataset at the participant and non-iid datasets among different participants; thus, overall datasets in FL are non-iid. We define: $\Gamma=\mathcal{L}^{*}-\sum_{i=1}^{N}\mathcal{L}_i^{*}$. $\Gamma$ quantifies the degree of non-iid, and it goes to zero as the data instances increase. 

\begin{theorem}\label{theorem}
Let $\alpha=\max\{8L/\mu,E\}$ and $T$ is SGD operations on a participant in FL then we obtain following relation of communication rounds $\mathcal{R}$ with desired precision $q_o$ and local epoch $E$. 
\begin{align}\nonumber
 \mathcal{R}&=\frac{1}{E}\Big[\frac{L}{2\mu^2 q_o}\Big(4B+\mu^2\alpha\mathbb{E}\|\mathbf{w}_1-\mathbf{w}^{*}\|^2\Big)+1-\alpha\Big],\\ \nonumber 
&\text{where } B =  \sum_{i=1}^{N} \epsilon_i^2 \Big(\mathbb{E}\|\mathbb{E}[\nabla \mathcal{L}_i(\phi_i^t,\mathbf{w}_i^t)]-\nabla\mathcal{L}_i(\mathbf{w}_i^t)\|^2\Big)\\ \label{e9}
& \qquad \qquad + 6L\Gamma + 8(E-1)^2 \Big(\|\mathbb{E}[\nabla\mathcal{L}_i(\phi_i^t,\mathbf{w}_i^t)]\|^2\Big).
\end{align}
\end{theorem}
\begin{proof}
 Let $\mathbf{w}_i^t$ be the WPM of $p_i$ at global iteration $t$, where $1\le i \le N$, $1 \le t \le \mathcal{R}$. Let $T$ denotes the number of SGD operations required for Fed-NL, $T=\{\mathcal{R}E| \mathcal{R}=1,2,\cdots\}$. Fed-NL activates all the $N$ participants at global iteration $t+1$ ($t+1\in T$). We can define the updates of Fed-NL involving partial participants using~\cite{li2019convergence} and introducing $\mathbf{u}_i^{t+1}$ to represent immediate result, as:
 \begin{align}
\mathbf{u}_i^{t+1}=\mathbf{w}_i^{t}-\eta^t \mathbb{E}[\nabla\mathcal{L}_i(\phi_i^t, \mathbf{w}_i^{t})],\\
\mathbf{w}_i^{t+1}=\left\{\begin{matrix}
\mathbf{u}_i^{t+1} & \text{ if } t+1 \notin T, \\ 
\sum_{i=1}^N \epsilon_i \mathbf{u}_i^{t+1} & \text{ if } t+1 \in T. 
\end{matrix}\right.  
 \end{align}

We consider two virtual sequences similar to~\cite{li2019convergence} to prove Theorem~\ref{theorem}, \textit{i.e.,} $\overline{\mathbf{u}}^t=\sum_{i=1}^N\epsilon_i\mathbf{u}_i^t$ and $\overline{\mathbf{w}}^t=\sum_{i=1}^{N}\epsilon_i \mathbf{w}_i^t$. We obtain $\overline{\mathbf{u}}^{t+1}$ from single step of SGD over $\overline{\mathbf{w}}^t$. However, $\overline{\mathbf{u}}^t$ and $\overline{\mathbf{w}}^t$ are inaccessible when $t+1\notin \mathcal{R}$. We define $\overline{\mathcal{G}}^t=\sum_{i=1}^{N}\epsilon_i \nabla\mathcal{L}_i(\mathbf{w}_i^{t})$ and $\mathcal{G}^t=\sum_{i=1}^{N}\epsilon_i \mathbb{E}[\nabla\mathcal{L}_i(\phi_i^t, \mathbf{w}_i^{t})]$. This implies $\overline{\mathbf{u}}^{t+1}=\overline{\mathbf{w}}^t-\eta^t\mathcal{G}^t$ and $\mathbb{E}[\mathcal{G}^t]=\overline{\mathcal{G}}^t$. We have $\overline{\mathbf{w}}^{t+1}=\overline{\mathbf{u}}^{t+1}$ irrespective of $t+1\in T$ or $\notin T$. 

Using Lemmas~$1$, $2$, and $3$ as given in~\cite{li2019convergence} and considering $\mathbf{w}^{*}$ as WPM to obtain $\mathcal{L}^{*}$, we get the following equations with added constraints of noisy labels in the dataset. 
\begin{align}\nonumber
 \mathbb{E}\|\overline{\mathbf{u}}^{t+1}-\mathbf{w}^{*}\|^2 & \le  (1-\eta^t)\mathbb{E}\|\overline{\mathbf{w}}^t-\mathbf{w}^{*}\|^2+\eta^2 \mathbb{E}\Big[\|\mathcal{G}^t-\overline{\mathcal{G}}^t\|^2\Big], \\ \label{l1}
  &+ 6L\Gamma(\eta^t)^2 + 2\mathbb{E}\sum_{i=1}^{N} \epsilon_i\| \overline{\mathbf{w}}^t-\mathbf{w}_i^t\|^2,
 \end{align} 
where, $\mathbb{E}\|\mathcal{G}^t-\overline{\mathcal{G}}^t\|^2  \le  \sum_{i=1}^N \epsilon_i \sigma_i^2$ and $\mathbb{E}\Big[\sum_{i=1}^{N}  \epsilon_i\| \overline{\mathbf{w}}^t  -\mathbf{w}_i^t\|^2 \Big] \le 4 (\eta^t)^2 (E-1)^2 G^2$. Let assume $\mathscr{S}^t = \mathbb{E}\|\overline{\mathbf{w}}^t  -\mathbf{w}_i^t\|^2$, we get:
\begin{align} \label{l4}
 \mathscr{S}^{t+1} & \le (1-\eta^t\mu)\mathscr{S}^t+ (\eta^t)^2 B,\\ \nonumber
&\text{where } B =  \sum_{i=1}^{N} \epsilon_i^2 \Big(\mathbb{E}\|\mathbb{E}[\nabla \mathcal{L}_i(\phi_i^t,\mathbf{w}_i^t)]-\nabla\mathcal{L}_i(\mathbf{w}_i^t)\|^2\Big)\\ \nonumber 
& \qquad \qquad + 6L\Gamma + 8(E-1)^2 \Big(\|\mathbb{E}[\nabla\mathcal{L}_i(\phi_i^t,\mathbf{w}_i^t)]\|^2\Big).
\end{align}

Convergence of FL relies on the diminishing learning rate~\cite{li2019convergence} and is as: $\eta^t=\frac{\theta}{t+\alpha}$. $\theta>1/\mu$ and $\alpha>0$ such that $\eta^1\le \min \Big(1/\mu, 1/4L\Big)$ and $\eta^t=2\eta^{t+E}$. We use induction method to prove $\mathscr{S}^t\le \frac{r}{\alpha+t}$, where $r=\max \Big\{ \frac{\theta^2 B}{\theta\mu-1}, (\alpha+1)\mathscr{S}^1\Big\}$. $r$ holds for $t=1$; thus, $\mathscr{S}^{t+1}$ satisfies following inequalities with $B$ given in (\ref{l4}) for noisy dataset.
\begin{align} \label{l5}
 \mathscr{S}^{t+1} & \le (1-\eta^t \mu) \mathscr{S}^t + (\eta^t)^2 B= \frac{r}{t+\alpha+1}.
\end{align}

\noindent Using L-smoothness, we obtain the precision at iteration $t$ as: 
\begin{equation}\label{l05}
\mathbb{E}[\mathcal{L}(\overline{\mathbf{w}}^t)]-\mathcal{L}^{*} \le \frac{L}{2} \mathscr{S}^{t} \le \frac{L}{2} \frac{r}{\alpha+t}.
\end{equation}

We consider $\theta=\frac{2}{\mu}$ and $\alpha=\max\{8L/\mu,E\}-1$ to minimize (\ref{l05}). We obtain $\eta^t=\frac{2}{\mu(\alpha+t)}$, where $\eta^t$ satisfies the following condition: $\eta^t \le 2 \eta^{t+E}$, $t\ge1$. Using $B$ given in (\ref{l4}) for noisy dataset, we get:
\begin{align} \label{e99}
 r=\max \Big\{ \frac{\theta^2 B}{\theta\mu-1}, (\alpha+1)\mathscr{S}^1\Big\} & 
 \le \frac{4B}{\mu^2} + (\alpha+1) \mathscr{S}^1. 
\end{align} 
Substituting $r$ from (\ref{e99}) into (\ref{l05}), we obtain,
\begin{align}\label{l6}
\mathbb{E}[\mathcal{L}(\overline{\mathbf{w}}^t)]-\mathcal{L}^{*} 
&\le \frac{L/2\mu^2}{\alpha+T-1}\Big(4B+\mu^2\alpha\mathbb{E}\|\mathbf{w}_1-\mathbf{w}^{*}\|^2\Big).
\end{align}

Using (\ref{l6}) we get the desired precision $q_o$ after $\mathcal{R}$ rounds as:
\begin{small}
\begin{align}\label{l7}
 q_o&=\mathbb{E}[\mathcal{L}(\mathbf{w}^\mathcal{R})]-\mathcal{L}^{*}  \le  \frac{L/2\mu^2}{\alpha+T-1}\Big(4B+\mu^2\alpha\mathbb{E}\|\mathbf{w}_1-\mathbf{w}^{*}\|^2\Big).
\end{align}
\end{small}

Similar as the convergence of FedAvg in~\cite{li2019convergence}, (\ref{l7}) indicates Fed-NL converges to a global optimum at the rate of $O(1/T)$. Using the upper bound of $q_o$ and $T=\mathcal{R}E$, we obtain:

\begin{equation}\label{e10}
 q_o = \frac{L/2\mu^2}{\alpha+T-1}\Big(4B+\mu^2\alpha\mathbb{E}\|\mathbf{w}_1-\mathbf{w}^{*}\|^2\Big).
 \end{equation}
Upon solving (\ref{e10}), we obtain communication round ($\mathcal{R}$) as follows: 
 \begin{equation}\label{e11}
 \mathcal{R}=\frac{1}{E}\Big[\frac{L}{2\mu^2 q_o}\Big(4B+\mu^2\alpha\mathbb{E}\|\mathbf{w}_1-\mathbf{w}^{*}\|^2\Big)+1-\alpha\Big].
\end{equation} 
Hence proved. 
\end{proof}

\section{Performance Evaluation}\label{evaluation}
This section first describes the collected and existing datasets (SHL~\cite{shl2}, HAR~\cite{anguita2013public}, FEMNIST~\cite{caldas2018leaf}, and CIFAR-10~\cite{krizhevsky2009learning}) and the baseline techniques for ablation studies. Next, we discuss implementation details and validation metrics. Finally, we conduct the experimental evaluation to validate the performance of the Fed-NL.

\subsection{Datasets and baseline techniques}

We collected a Locomotion Mode Recognition (LMR) dataset to recognize six modes, \textit{i.e.},  bicycle, bike, car, auto-rickshaw, bus, and train. We developed an android application to facilitate data collection that uses three onboard sensors, including accelerometer, gyroscope, and magnetometer. We appointed $40$ volunteers during LMR data collection. The data instances were recorded for $60$ seconds, which is enough time to provide sufficient distinguishable characteristics among different locomotion modes. Volunteers were instructed to collect $200$ data instances per locomotion mode during data collection. Thus, we obtain $48000$ data instances in the LMR dataset, where $40$ volunteers $\times 6$ locomotion modes $\times 200$ repetitions $=48000$. To automatically annotate (or tag) the data instances, the volunteers were instructed to capture an image of the locomotion mode before recording the data. The volunteers also tag the data instances manually to check the correctness of the automatic annotation. Using manual tagging, we observed that the automated annotation introduced noisy labels distributed asymmetrically, generating the need to handle such labels. Next, we considered the existing SHL~\cite{shl2} dataset, which used similar sensors as LMR (\textit{i.e.,} accelerometer, gyroscope, and magnetometer) and possessed instances for more classes than LMR. Afterward, we considered HAR~\cite{anguita2013public} dataset due to its broader applicability in FL literature during experiments. HAR also utilized a similar set of sensors (accelerometer, gyroscope, and magnetometer) as the LMR dataset for human activities recognition. Finally, we considered widely used image-based FEMNIST~\cite{caldas2018leaf} and CIFAR-10~\cite{krizhevsky2009learning} datasets in FL. These datasets were selected due to free accessibility, real-life acquisition, and correct annotations. FEMNIST is built by dividing the existing Extended MNIST~\cite{cohen2017emnist} depending upon the writer of the digit/character. CIFAR-10 dataset~\cite{krizhevsky2009learning} comprises $60000$ images of $10$ different classes. Each class has $6000$ images.

Since the collected and existing datasets have noise-free labels; thus, we explicitly add asymmetric noisy labels in the dataset to test Fed-NL in the real world. To inject noisy asymmetric labels into the dataset, we follow the label transition to design some of the structure of real mistakes for similar classes, as discussed in~\cite{chen19g}.


Section~\ref{related-work} illustrated that most of the existing literature on FL to handle noisy labels are considered optimizing weighted contribution~\cite{9412599} and reducing noisy label instances~\cite{yang2020robust,chen2020focus} techniques. We selected techniques~\cite{9412599, yang2020robust,chen2020focus, mcmahan2017} as baselines for ablation studies, denoted as Base1~\cite{9412599}, Base2~\cite{yang2020robust}, Base3~\cite{chen2020focus} and Base4~\cite{mcmahan2017}, to compare and evaluate the effectiveness of Fed-NL. 

\subsection{Implementation details}
We used the DeepZero model given in~\cite{9164991} during the experimental evaluation on LMR and SHL datasets. We used Convolutional Neural Networks (CNN) as the classifier for evaluating Fed-NL on HAR, FEMNIST, and CIFAR-10 datasets. CNN model is similar to the one used in~\cite{wang2019adaptive}. We selected $40$ devices as participants in Fed-NL. We combined and randomly shuffled the dataset collected from $40$ volunteers and split the shuffled data into $40$ parts to obtain non-iid datasets for $N=40$ participants. Similarly, we divide the existing datasets (SHL, HAR, FEMNIST, and CIFAR-10) into $40$ parts for all the participants. We set $\mathcal{L}^{*}=0.7999$, $E=20$, $N=40$, and used other parameters similar as given in~\cite{li2019convergence} to solve (\ref{e11}). We store noise-free collected and existing datasets on the server with all instances; thus, the mismatch between the distribution of participants and the server's dataset did not persist. The considered participating devices are smartphones with similar RAM and processing capacity. They are operating over the same Wi-Fi connection. We used a Dell server with an Intel Dual Xeon processor operating over a Gigabit Ethernet connection as server.

\subsection{Validation metrics}
This work used standard metrics to evaluate the performance of Fed-NL: F1-score and accuracy. We also derive the following metrics:

\noindent $\bullet$ \textit{Local performance metric:} This metric determines the performance of participant $p_i$ on its dataset $\mathcal{D}_i$, \textit{i.e., } F1-score, accuracy and confusion metric on testing sub-dataset of $\mathcal{D}_i$. It answers the following question: \textit{how much improvement or deterioration in the performance of participants are observed by using the Fed-NL}?\\
\noindent $\bullet$ \textit{Global performance metric:} It estimates the performance of participants on server dataset $\mathcal{D}$. It answers: \textit{how well the trained model on the participant performs on the server's dataset}? 

\subsection{Experimental results}

\subsubsection{Impact of noise ratio}\label{e01}
This experiment aims to determine the impact of noise ratio on the local and global performance of a participant in FL. During the experiment, we considered the proposed Fed-NL and the baseline techniques (Base1, Base2, Base3, Base4). We used the asymmetric noisy labels on the collected LMR dataset.

Table~\ref{table1}(a) illustrates the local performance of a participant at different noise ratios using the collected LMR dataset and the communication rounds $=120$. It shows that Fed-NL gives higher accuracy than Base1, Base2, Base3, and Base4 for all considered ratios. Base1 removed noisy samples from the dataset, reducing the training data size. Base2 and Base3 did not eliminate the data instances with noisy labels from their dataset; thus, the model learned with noisy label instances. Base4 did not incorporate any noise handling mechanism; therefore, it achieved the least performance. Fed-NL removed the data instances with noisy labels and added some of the correctly annotated instances from the server during normalization; thus, it achieved higher performance. We also observed that Base2 achieved comparable performance to Fed-NL due to the efficient grouping of the class labels using centroids of participants and server. The low noise ratio achieved higher performance, which gradually decreased with increasing noise ratio. It is because fewer instances with correctly annotated labels were available to train the model on a higher noise ratio.  

Table~\ref{table1}(a) also illustrates the F1-score is greater than the accuracy, and the difference between accuracy and F1-score increased with the noise ratio. Because with the increase in the noise ratio, the false negatives and false positives are crucial and give less accuracy. We also observed that the difference between accuracy and F1-score is highest at noise ratio $=0.6$ because the role of true positive and true negative is least due to noisy training instances, which reduced accuracy. Similar observations can be made for global performance, as shown in Table~\ref{table1}(b). We repeated the same experiment for all the participants with a given noise ratio. We observed that the standard deviation in F1-score and accuracy of $p_i$ with the average F1-score and accuracy of all the participants are $\pm 0.05$ and $\pm 0.06$, respectively. It indicates Fed-NL performed effectively on each participant.    

\noindent \textit{Observation: The first observation from the result is that the noise ratio deteriorates local and global performance. Fed-NL achieves high performance by removing instances with noisy labels and adding correctly annotated instances from the server. Next, we observed that testing on personalized data instances achieves higher performance than testing on generalized data instances. Finally, F1-score is more effective in noisy environments than accuracy.}

\begin{table}[h]
\caption{Impact of the noise ratio on the local and global performance of $p_i$ using LMR with asymmetric distribution of noisy labels and $\mathcal{R}=120$. Acc. and F1 are accuracy and F1-score, respectively.}
\label{table1}
\vspace{-0.3cm}
\resizebox{.48\textwidth}{!}{
\begin{tabular}{|ccc|cccccc|}
\multicolumn{9}{c}{(a) Local performance}\\ \hline
\multicolumn{3}{|l|}{\textbf{Noise ratio} ($\beta_i$)}                                                                    & $\mathbf{0.1}$   & $\mathbf{0.2}$   & $\mathbf{0.3}$   & $\mathbf{0.4}$   & $\mathbf{0.5}$   & $\mathbf{0.6}$      \\ \hline
\multicolumn{1}{|l|}{\multirow{10}{*}{\rotatebox{90}{\textbf{Approach/Technique}}}} & \multicolumn{1}{l|}{\multirow{2}{*}{Fed-NL}} & F1 & $92.23$ & $88.53
$ & $81.43$ & $74.76$ & $65.43$ & $52.23$  \\ \cline{3-9} 
\multicolumn{1}{|l|}{}                            & \multicolumn{1}{l|}{}   & Acc. & $91.43$ & $86.67$ & $77.23$ & $69.92$  & $57.39$ & $43.82$ \\ \cline{2-9} 

\multicolumn{1}{|l|}{}                            & \multicolumn{1}{l|}{\multirow{2}{*}{Base1~\cite{9412599}}}     & F1 & $85.43$ & $81.57$ & $75.27$ & $67.59$ & $55.37$       & $42.47$              \\ \cline{3-9} 
\multicolumn{1}{|l|}{}                            & \multicolumn{1}{l|}{}                        & Acc. & $84.21$ & $79.17$ & $71.82$  & $62.13$  & $49.10$  & $35.27$             \\ \cline{2-9} 
\multicolumn{1}{|l|}{}                            & \multicolumn{1}{l|}{\multirow{2}{*}{Base2~\cite{yang2020robust}}}     & F1 &  $90.54$  & $86.13$ & $78.53$ & $72.11$ & $62.47$      & $50.21$             \\ \cline{3-9} 
\multicolumn{1}{|l|}{}                            & \multicolumn{1}{l|}{}                        & Acc. & $89.22$  & $84.31$ & $73.44$ & $65.21$ & $53.21$ & $41.31$             \\ \cline{2-9} 
\multicolumn{1}{|l|}{}                            & \multicolumn{1}{l|}{\multirow{2}{*}{Base3~\cite{chen2020focus}}}     & F1 & $88.37$ & $83.78$ & $76.23$  & $71.21$  & $60.32$      & $48.59$             \\ \cline{3-9} 
\multicolumn{1}{|l|}{}                            & \multicolumn{1}{l|}{}                        & Acc. & $86.51$ & $80.25$ & $72.42$ & $64.39$ & $51.23$ &   $39.47$                 \\ \cline{2-9} 
\multicolumn{1}{|l|}{}                            & \multicolumn{1}{l|}{\multirow{2}{*}{Base4~\cite{mcmahan2017}}}     & F1 & $84.81$ & $80.19$ & $72.23$  & $64.21$  & $53.62$      & $41.53$             \\ \cline{3-9} 
\multicolumn{1}{|l|}{}                            & \multicolumn{1}{l|}{}                        & Acc. & $83.54$ & $78.86$ & $68.34$ & $61.57$ & $48.31$ &   $34.71$                 \\ \hline 
\end{tabular}
}
\resizebox{.48\textwidth}{!}{
\begin{tabular}{|ccc|cccccc|}
\multicolumn{9}{c}{(b) Global performance}\\ \hline
\multicolumn{3}{|l|}{\textbf{Noise ratio} ($\beta_i$)}                                                                    & $\mathbf{0.1}$   & $\mathbf{0.2}$   & $\mathbf{0.3}$   & $\mathbf{0.4}$   & $\mathbf{0.5}$   & $\mathbf{0.6}$      \\ \hline
\multicolumn{1}{|l|}{\multirow{10}{*}{\rotatebox{90}{\textbf{Approach/Technique}}}} & \multicolumn{1}{l|}{\multirow{2}{*}{Fed-NL}} 
& F1 & $85.49$ & $80.43 $ & $74.92$ & $67.44$ & $58.21$ & $39.17$  \\ \cline{3-9} 
\multicolumn{1}{|l|}{}   & \multicolumn{1}{l|}{}   
& Acc. & $83.22$ & $77.52$ & $68.63$ & $61.29$  & $50.47$ & $30.03$ \\ \cline{2-9} 

\multicolumn{1}{|l|}{}     & \multicolumn{1}{l|}{\multirow{2}{*}{Base1~\cite{9412599}}}     
& F1 & $80.22$ & $73.92$ & $67.58$ & $58.42$ & $49.59$       & $30.22$              \\ \cline{3-9} 
\multicolumn{1}{|l|}{}                            & \multicolumn{1}{l|}{}                        
& Acc. & $78.49$  & $71.07$ & $60.29$ & $52.49$  & $41.29$   & $21.37$             \\ \cline{2-9} 
\multicolumn{1}{|l|}{}                            & \multicolumn{1}{l|}{\multirow{2}{*}{Base2~\cite{yang2020robust}}}     
& F1 &  $84.23$  & $78.55$ & $72.71$ & $63.92$ & $55.87$      & $36.34$             \\ \cline{3-9} 
\multicolumn{1}{|l|}{}                            & \multicolumn{1}{l|}{}                        
& Acc. &$82.81$ & $76.32$     & $70.11$ & $57.29$ & $46.38$  & $24.32$                    \\ \cline{2-9} 
\multicolumn{1}{|l|}{}                            & \multicolumn{1}{l|}{\multirow{2}{*}{Base3~\cite{chen2020focus}}}     
& F1 & $82.34$ & $76.91$ & $70.29$  & $60.59$  & $53.31$      & $33.82$             \\ \cline{3-9} 
\multicolumn{1}{|l|}{}                            & \multicolumn{1}{l|}{}                        
& Acc. & $80.93$   & $75.17$ & $67.36$ & $51.27$   & $43.29$      & $22.81$                    \\ \cline{2-9} 
\multicolumn{1}{|l|}{}                            & \multicolumn{1}{l|}{\multirow{2}{*}{Base4~\cite{mcmahan2017}}}     & F1 & $77.67$ & $71.53$ & $63.42$  & $56.93$  & $46.73$      & $27.21$             \\ \cline{3-9} 
\multicolumn{1}{|l|}{}  & \multicolumn{1}{l|}{}    & Acc. & $76.31$ & $69.42$ & $57.51$ & $50.71$ & $39.36$ &   $20.72$                 \\ \hline 
\end{tabular}
}
\end{table}

\subsubsection{Impact of noise ratio normalization}\label{e02}
This experiment aims to highlight the requirement of noise ratio normalization. We considered the collected LMR dataset, asymmetric noisy labels, noise ratio $0.2$, communication rounds $\mathcal{R}=120$ using (\ref{e11}), and local F1-score metric during the experiment. We used notations $\mathbf{a_1}$, $\mathbf{a_2}$, $\mathbf{a_3}$, $\mathbf{a_4}$, $\mathbf{a_5}$, and $\mathbf{a_6}$ to denote class labels bicycle, bike, car, auto-rickshaw, bus, and train of LMR dataset, respectively.

Figure~\ref{confusion}(a) illustrates the confusion matrix of the collected LMR dataset, where the noise ratio is un-normalized. The local F1-score of class $\mathbf{a_3}$ (''car'') is highest and class $\mathbf{a_5}$ (''bus'') is lowest. It is because the training instances of class $\mathbf{a_3}$ are less hampered with asymmetric distribution of noisy labels and vice-versa for class $\mathbf{a_5}$. The overall local F1-score without normalization is around $83\%$, less than the normalization case of around $89\%$ at a noise ratio of $0.2$. Figure~\ref{confusion}(b) depicts the confusion matrix with normalized noise ratio on the participant's dataset. The result illustrates that the normalization of the noise ratio improves the F1-score of all the classes in the dataset. The improvement in the F1-score of class $\mathbf{a_3}$ is minimum ($\approx 1.0\%$) as its instances are less hampered during noise insertion. However, the improvement on the F1-score of $\mathbf{a_1}$ is highest around $10\%$. During normalization, the participant gets larger correctly annotated instances from the server dataset for $\mathbf{a_1}$.\\
\noindent \textit{Observation: Noise ratio normalization is crucial in improving class-wise performance. It ensures the uniform distribution of noisy labels among different classes in the dataset. Fed-NL uses noise ratio normalization to exploit the correctly annotated data instances of the server's dataset to improve the performance.}

\begin{figure}[h]
        \centering
\begin{tabular}{cc}
\hspace{-.2cm}\includegraphics[scale=0.41]{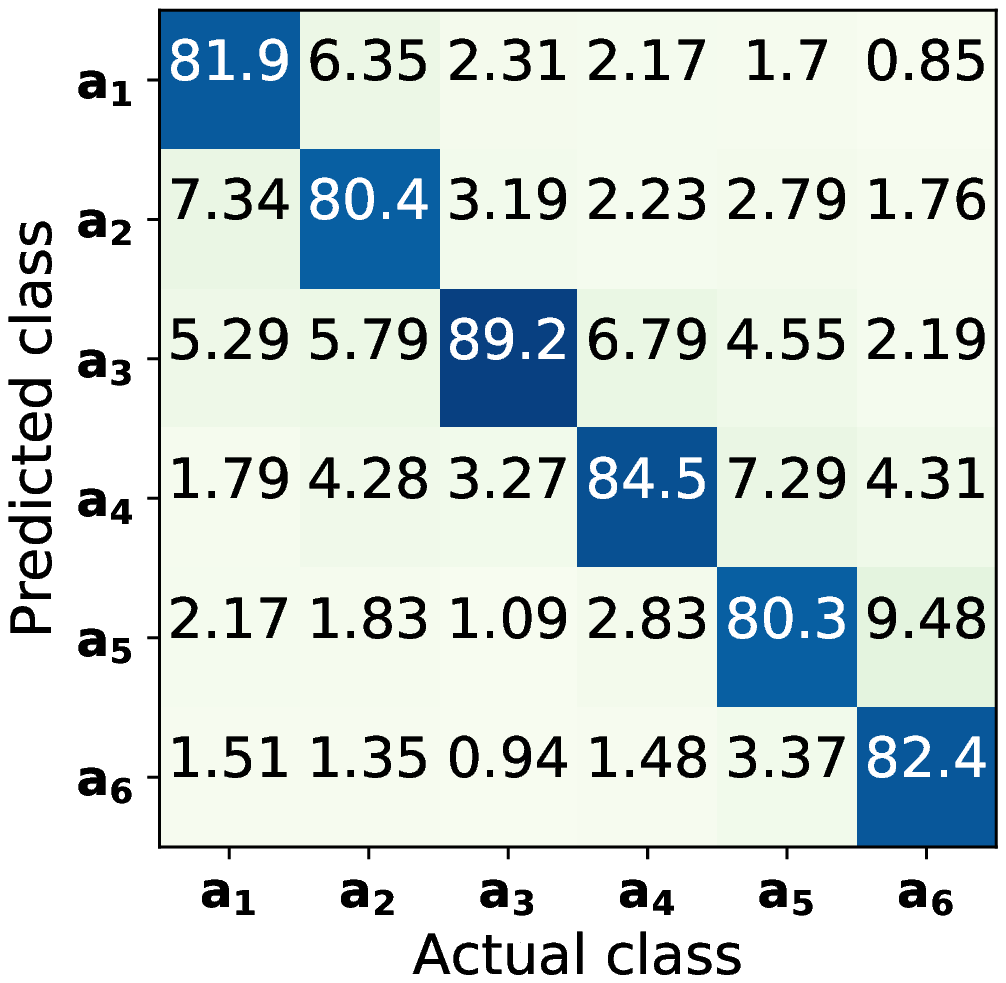}  &\hspace{-.1cm}   \includegraphics[scale=0.41]{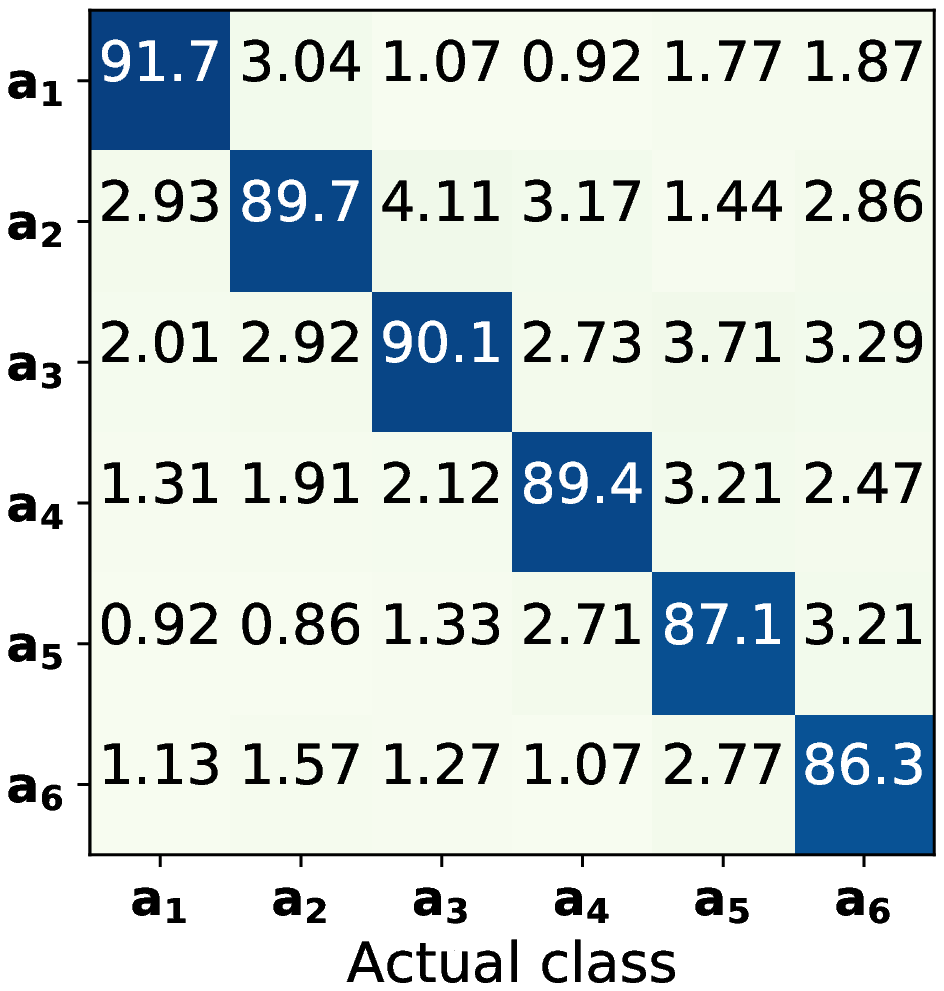}  \\ 
\small{(a) Un-normalized noise ratio.} &   \small{(b) Normalized noise ratio.} \\  
\end{tabular}
\vspace{-0.4cm}
        \caption{Impact of noise ratio normalization on local F1-score of a participant using collected LMR dataset and noise ratio $0.2$. } 
        \label{confusion}
\end{figure}

\subsubsection{Impact of dataset}\label{e03}
This experiment aims to determine the impact of LMR, SHL, HAR, FEMNIST, and CIFAR-10 datasets on performance. We considered the asymmetric noisy labels in datasets and rounds $\mathcal{R}$ as (Dataset $\rightarrow \mathcal{R}$): LMR $\rightarrow 120$, SHL$\rightarrow 120$, HAR$\rightarrow 100$, FEMNIST$\rightarrow 100$, and CIFAR-$10$$\rightarrow 150$.   

Table~\ref{dataset}(a) illustrates local F1-score decreases with the noise ratio for all the datasets. The reason behind this decrement is mentioned in Section~\ref{e01}. The F1-score of the SHL dataset decreases much more rapidly than that of the collected LMR. It is because the class labels in SHL possess imbalanced class labels that result in the non-uniform distribution of noisy labels, and performance is compromised. At a low noise ratio, \textit{e.g.,} $0.1$, the local F1-score on LMR and SHL datasets are comparable, as shown in Table~\ref{dataset}(a).  However, at a high noise ratio, \textit{e.g.,} $0.6$, the local F1-score on LMR is significantly higher than the SHL dataset. The impact of unbalanced class labels at a low noise ratio is less, which increases with the increase in the ratio. HAR, FEMNIST, and CIFAR-10 datasets have a large number of instances; thus, the model on these datasets achieves a higher F1-score, in contrast with LMR and SHL datasets. Moreover, Fed-NL achieved F1-score $>80\%$ even at the noise ratio of $0.6$ on these datasets. It is because the noisy labels are successfully handled by demanding data instances from the server. Similarly, Table~\ref{dataset}(b) illustrates the impact of different datasets on global F1-score.

\noindent \textit{Observation: An interesting observation from the result is that the unbalanced classes in the dataset with noisy labels result in higher performance compromise during training. It also indicates that if data instances of participants are noisy with unbalanced class labels, achieving higher performance is tedious. Therefore, an unbalanced dataset at the participant demands a significant amount of server data instances during normalization. Further, when the number of training instances in the dataset is high, the achieved F1-score is high. }

\begin{table}[h]
\caption{Illustration of impact of different datasets, collected LMR, SHL, HAR, FEMNIST, CIFAR-10, on local and global F1-scores.}
\label{dataset}
\centering
\resizebox{.47\textwidth}{!}{
\begin{tabular}{|cc|cccccc|}
\multicolumn{8}{c}{(a) Local F1-score}\\ \hline
\multicolumn{2}{|l|}{\textbf{Noise ratio}}                                                                    & $\mathbf{0.1}$   & $\mathbf{0.2}$   & $\mathbf{0.3}$   & $\mathbf{0.4}$   & $\mathbf{0.5}$   & $\mathbf{0.6}$      \\ \hline
\multicolumn{1}{|l}{\multirow{5}{*}{\rotatebox{90}{\textbf{Dataset}}}} 
& \multicolumn{1}{|l|}{\multirow{1}{*}{LMR}}  & $92.23$ & $88.53 $ & $81.43$ & $74.76$ & $65.43$ & $52.23$  \\ \cline{2-8} 
& \multicolumn{1}{|l|}{\multirow{1}{*}{SHL}}  & $91.31$ & $87.61 $ & $78.33$ & $69.87$ & $61.43$ & $48.07$  \\ \cline{2-8} 
& \multicolumn{1}{|l|}{\multirow{1}{*}{HAR}}  & $94.20$ & $92.64$  & $90.61$ & $88.93$ & $86.29$ & $84.11$   \\ \cline{2-8} 
& \multicolumn{1}{|l|}{\multirow{1}{*}{FEMNIST}}& $95.10$ & $93.27$  & $91.03$ & $89.20$ & $87.19$ & $86.37$   \\ \cline{2-8} 
& \multicolumn{1}{|l|}{\multirow{1}{*}{CIFAR-10}}  & $92.71$ & $90.39 $ & $89.71$ & $87.76$ & $85.73$ & $82.91$  \\ \hline
\end{tabular}
}
\resizebox{.47\textwidth}{!}{
\begin{tabular}{|cc|cccccc|}
\multicolumn{8}{c}{(b) Global F1-score}\\ \hline
\multicolumn{2}{|l|}{\textbf{Noise ratio}}                                                                    & $\mathbf{0.1}$   & $\mathbf{0.2}$   & $\mathbf{0.3}$   & $\mathbf{0.4}$   & $\mathbf{0.5}$   & $\mathbf{0.6}$      \\ \hline
\multicolumn{1}{|l}{\multirow{5}{*}{\rotatebox{90}{\textbf{Dataset}}}} 
& \multicolumn{1}{|l|}{\multirow{1}{*}{LMR}}  & $85.49$ & $80.43 $ & $74.93$ & $67.45$ & $58.21$ & $39.73$  \\ \cline{2-8} 
& \multicolumn{1}{|l|}{\multirow{1}{*}{SHL}}  & $84.25$ & $78.61 $ & $71.63$ & $64.81$ & $55.21$ & $34.07$  \\ \cline{2-8} 
& \multicolumn{1}{|l|}{\multirow{1}{*}{HAR}}  & $87.41$ & $84.51 $ & $80.93$ & $77.91$ & $75.73$ & $72.46$  \\ \cline{2-8} 
& \multicolumn{1}{|l|}{\multirow{1}{*}{FEMNIST}}  & $88.26$ & $86.23 $ & $82.61$ & $79.29$ & $77.43$ & $74.71$  \\ \cline{2-8} 
& \multicolumn{1}{|l|}{\multirow{1}{*}{CIFAR-10}}  & $86.22$ & $83.51 $ & $79.81$ & $77.32$ & $74.17$ & $71.44$  \\ \hline
\end{tabular}
}
\end{table} 

\subsubsection{Impact of weighted contribution}\label{e04}
The objective of this experiment is to illustrate the effectiveness of the Fed-NL in estimating the weighted contribution of the noisy participants. In doing so, we considered three schemes $\mathbf{S_1}$ (Base4 (FedAvg~\cite{mcmahan2017})), $\mathbf{S_2}$ (Base3~\cite{chen2020focus}), and $\mathbf{S_3}$ (Fed-NL). In this work, we assigned $2400$ instances of collected LMR and $1631$ instances of SHL on all $40$ participants' devices for local training. Noisy labels are explicitly inserted in the dataset of the participant. We also assumed $20$ out of $40$ participants have noisy labels, \textit{i.e.,} $50\%$ participants are noisy. We considered contribution ratio calculated as: \textit{contribution on $x$ noise ratio in dataset/contribution on noise-free dataset}, where $x=\{0.1, 0.2, \cdots, 0.6\}$. 

Figure~\ref{datasets}(a1) illustrates the weighted contribution of a participant at different noise ratios on the considered schemes, \textit{i.e.,} $\mathbf{S_1}$, $\mathbf{S_2}$, and $\mathbf{S_3}$ using LMR dataset. It shows that the weighted contribution of Scheme $\mathbf{S_1}$ is the same for any noise ratio. It is because the weighted contribution of $\mathbf{S_1}$ depends only on the number of local data instances. However, $\mathbf{S_2}$ and $\mathbf{S_3}$ schemes reduce the contribution with the increase in noise ratio. $\mathbf{S_2}$ used cross loss among server and participant to determine weighted contribution, whereas $\mathbf{S_3}$ (proposed) used noise ratio and participant influence to determine contribution. Figure~\ref{datasets}(a2) depicts the reduction in local F1-score due to the increase in noise ratio. The result illustrates that $\mathbf{S_3}$ and $\mathbf{S_1}$ achieved the highest and lowest performance, respectively. It is because the $\mathbf{S_1}$ did not incorporate any mechanism to handle noisy labels, whereas $\mathbf{S_3}$ used noise ratio normalization and effectively determined contribution. Figures~\ref{datasets}(b1) and~\ref{datasets}(b2) illustrate the contribution and local F1-score on SHL dataset, respectively. 

\noindent \textit{Observation: An interesting observation from the result is that estimating the optimal weighted contribution at different noise ratios is crucial to achieving significant performance. It also indicates that considering the influence of participant and noise ratio while assessing weighted contribution helped in improving the performance.}

\begin{figure}[h]
         \centering
\begin{tabular}{cc}
\hspace{-0.2cm}\includegraphics[scale=0.4]{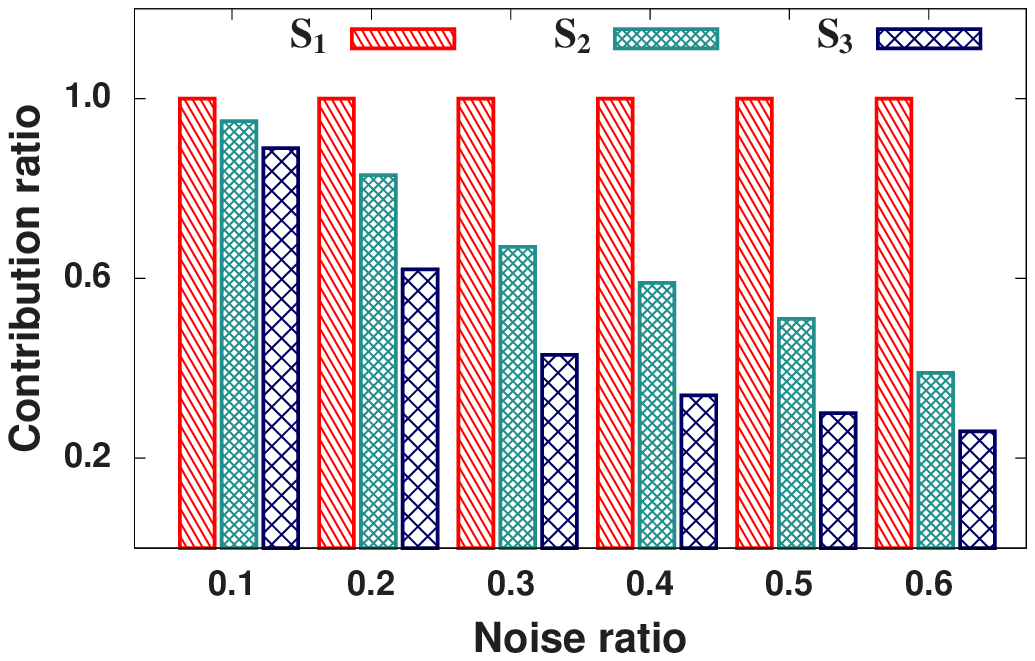}  & \hspace{-0.3cm} \includegraphics[scale=0.4]{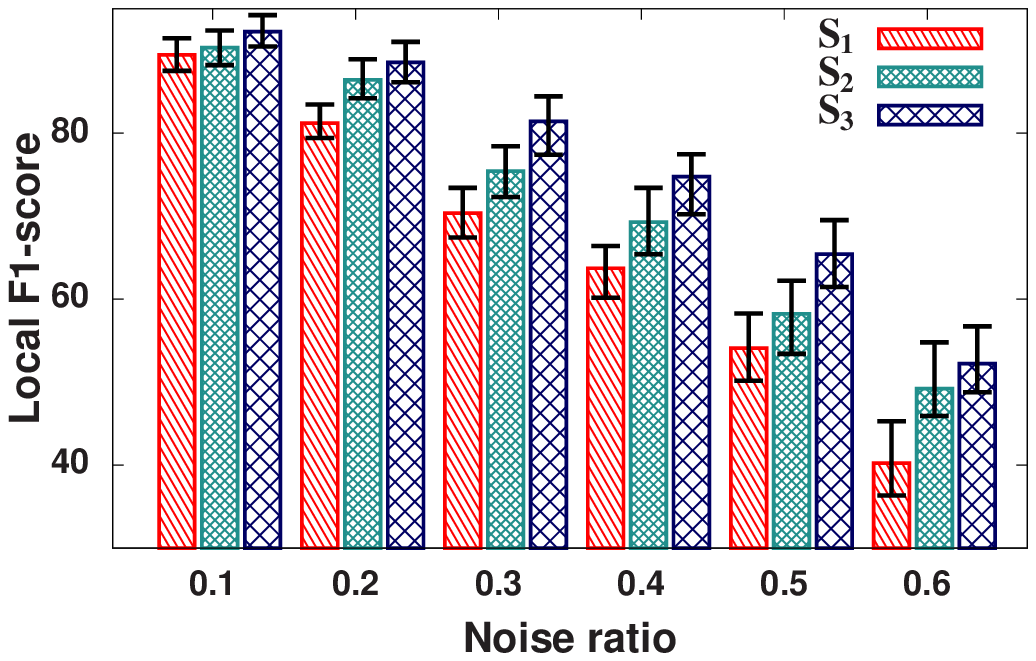} \\
\small{(a1) Contribution.} & \hspace{0.3cm} \small{(a2) Local F1-score.}\\
\multicolumn{2}{c}{\small{(a) LMR dataset.}}\\ 
\hspace{-0.2cm}\includegraphics[scale=0.4]{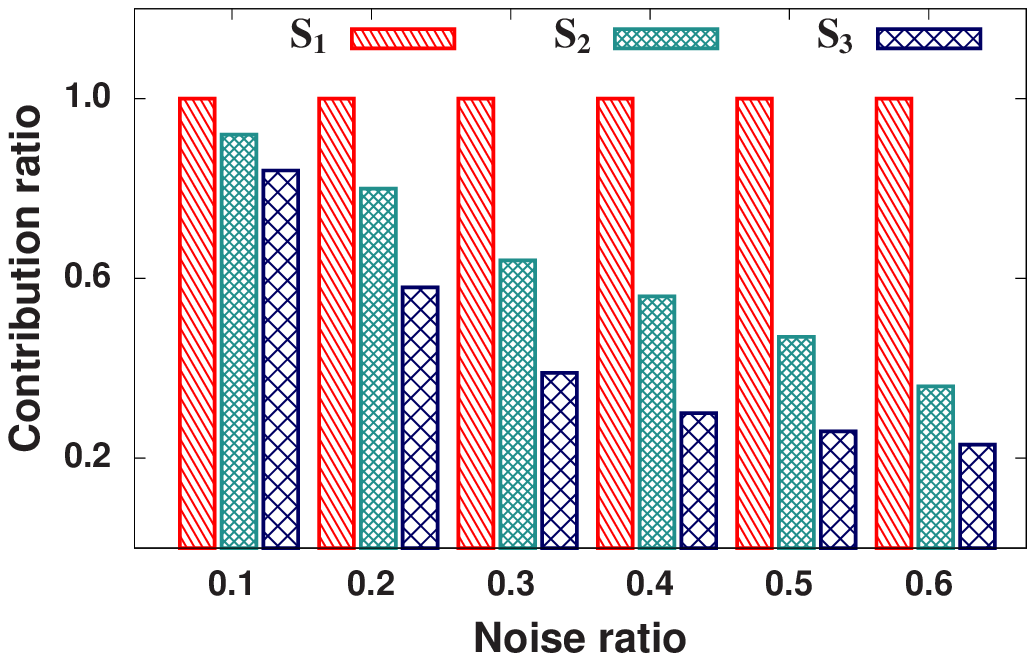}  & \hspace{-0.3cm} \includegraphics[scale=0.4]{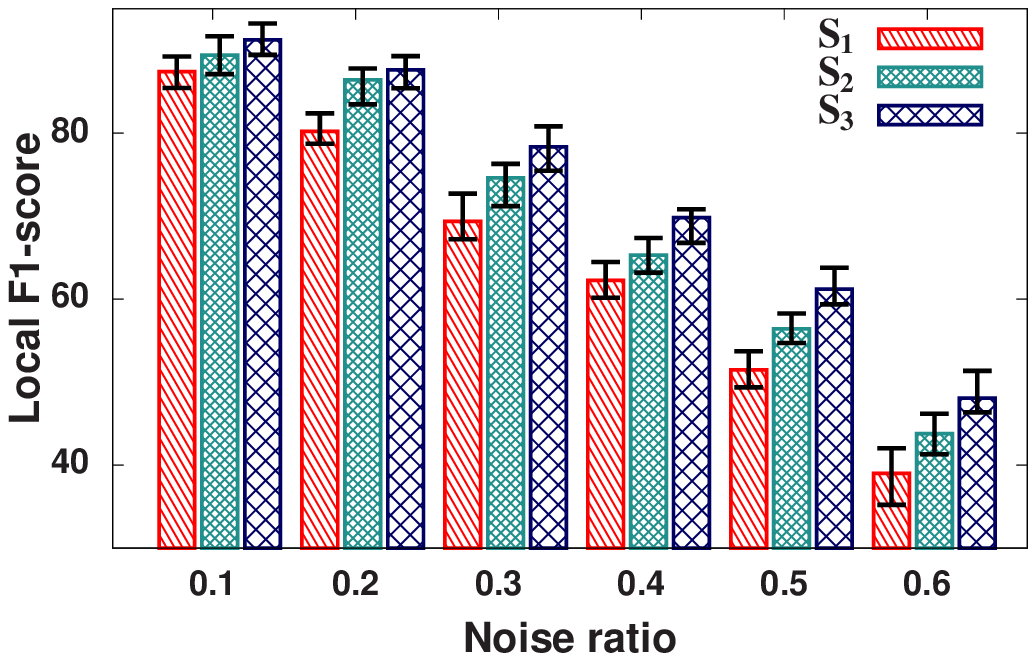} \\
\small{(b1) Contribution.} & \hspace{0.3cm} \small{(b2) Local F1-score.}\\
\multicolumn{2}{c}{\small{(b) SHL dataset.}}\\
\end{tabular}
\vspace{-0.5cm}
         \caption{Impact of weighted contribution on the datasets.} 
         \label{datasets}
         \vspace{-0.5cm}
 \end{figure}

\subsubsection{Number of communication rounds}\label{e05}
The objective of this experiment is to estimate the number of communication rounds (or global iterations) required for convergence of Fed-NL, Base1~\cite{9412599}, Base2~\cite{yang2020robust}, Base3~\cite{chen2020focus}, and Base4 (FedAvg~\cite{mcmahan2017}). We demonstrate the improvement in local accuracy by increasing the communication rounds on the collected LMR and SHL datasets at the noise ratio of $0.1$. We further extend the experiment to study the impact of noise ratio on communication rounds using the LMR dataset.

Figure~\ref{comm_round}(a) illustrates the impact of the communication round on the local accuracy of the participant using the collected LMR dataset. We observed a rapid increment in the local accuracy of the Fed-NL up to $43$ epochs and a marginal increment afterward. Later, we observed that after $60$ epochs, the performance improvement is negligible. It indicates that the Fed-NL converged after $60$ epochs and achieved an accuracy of $91.43\%$ at the noise ratio of $0.1$. Similarly, Base1 converged after $80$ epochs and achieved an accuracy of $84.21\%$. Base2 and Base3 acquired the local accuracy of $89.22\%$ and $86.51\%$, respectively, upon convergence. Base4 (FedAvg) required the highest number of communication rounds and achieved the least accuracy due to the absence of the noise handling mechanism. The higher performance and least epochs for convergence of Fed-NL are due to effective estimation of weighted contribution and noise ratio normalization via server dataset. Similar observation can be made for the SHL dataset, as shown in Figure~\ref{comm_round}(b).

\begin{figure}[h]
         \centering
\begin{tabular}{cc}
\hspace{-0.2cm}\includegraphics[scale=0.45]{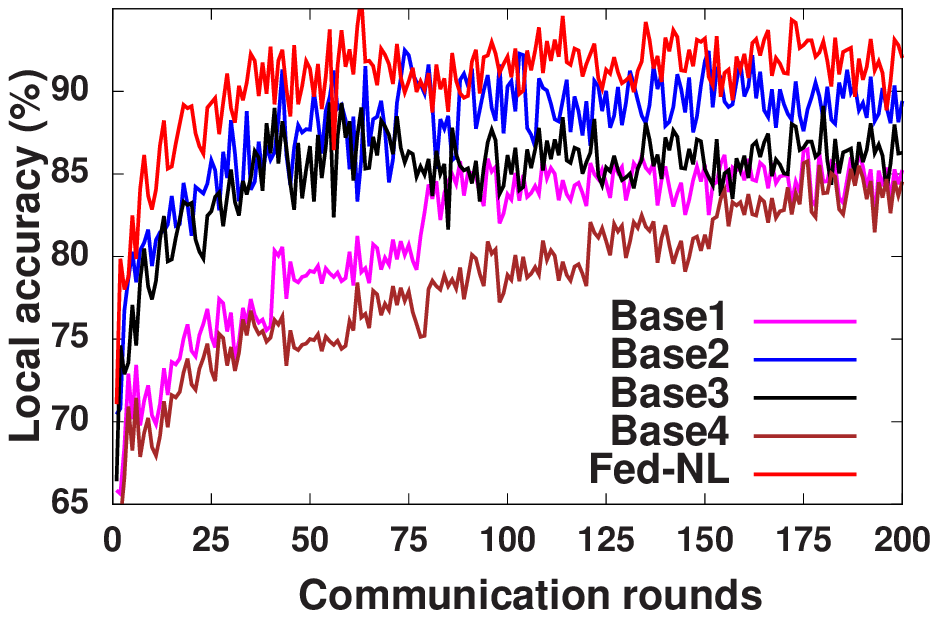}  & \hspace{-0.3cm} \includegraphics[scale=0.45]{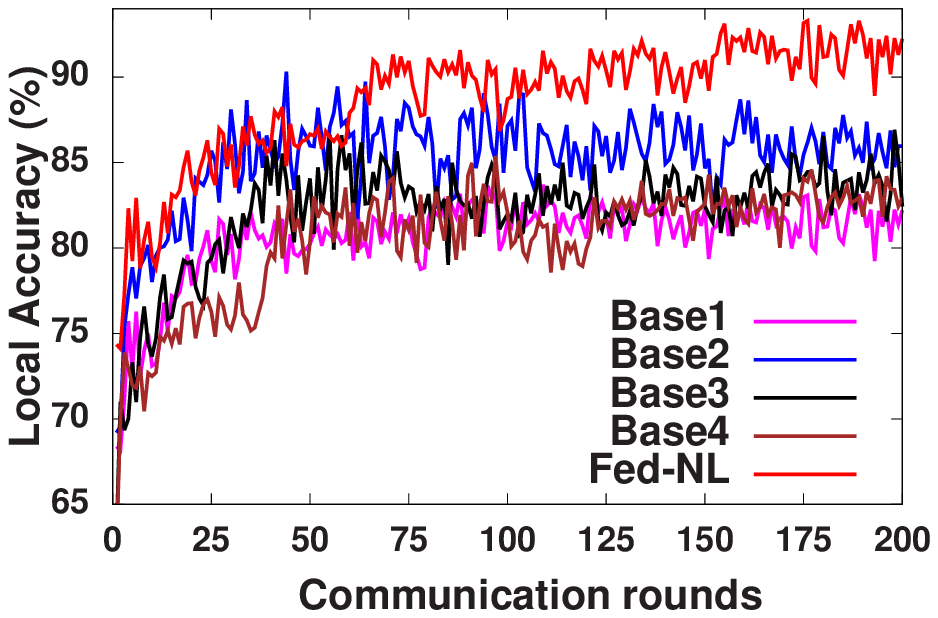} \\
\small{(a) LMR dataset.} & \hspace{0.3cm} \small{(b) SHL dataset.}\\
\end{tabular}
\vspace{-0.5cm}
         \caption{Communication rounds on Fed-NL, Base1~\cite{9412599}, Base2~\cite{yang2020robust}, Base3~\cite{chen2020focus}, and Base4 (FedAvg~\cite{mcmahan2017}) on LMR and SHL datasets.} 
         \label{comm_round}
         \vspace{-0.5cm}
 \end{figure}

Figure~\ref{comm_round1}(a) illustrates the impact of noise ratio on the global iterations (or communication rounds) using the LMR dataset. We fixed the value of desired precision $q_o=0.01$ and checked the efficiency of Fed-NL on different values of local epochs, \textit{i.e.,} $E=\{20, 40, 60\}$. The result shows that the increase in noise ratio increases the communication rounds, and a much higher increment is observed for the low value of local epochs. The result also demonstrates the communication round curve follows power-law till the noise ratio of $0.5$ and shows a linear increment afterward. The model converges to a sub-optimal state when the noise ratio is more than $0.5$. We further fixed the value of the local epoch to $20$. We estimated the desired precision on a fixed number of communication rounds ($100, 200, 300$), as shown in Figure~\ref{comm_round1}(b). We observe a similar pattern as Figure~\ref{comm_round1}(a), where desired precision increases with noise ratio.  

\noindent \textit{Observation: Noise ratio normalization and optimal weighted contribution are important factors to achieve high performance and minimize communication rounds. Using server dataset with noise-free labels mitigate the negative impact of the noisy labels in the dataset of the participants. Desired precision and local epochs play a decisive role while estimating the communication rounds for the convergence.}

\begin{figure}[h]
         \centering
\begin{tabular}{cc}
\hspace{-0.2cm}\includegraphics[scale=0.53]{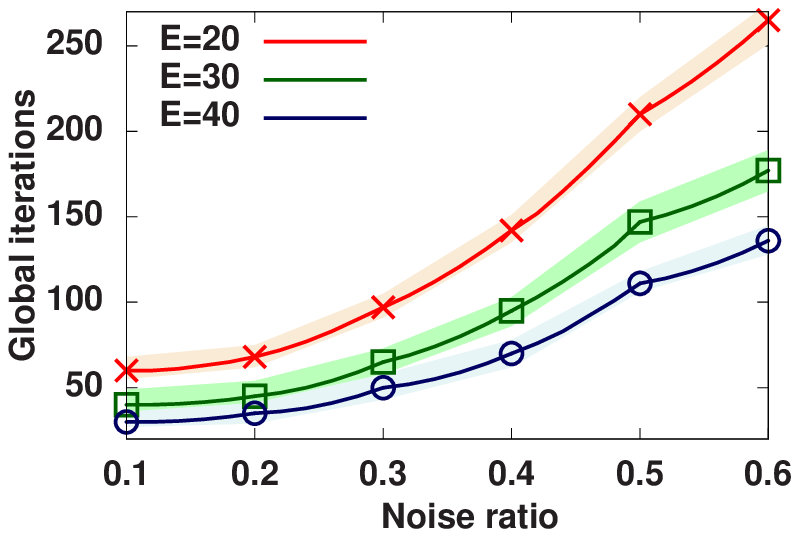}  & \hspace{-0.3cm} \includegraphics[scale=0.53]{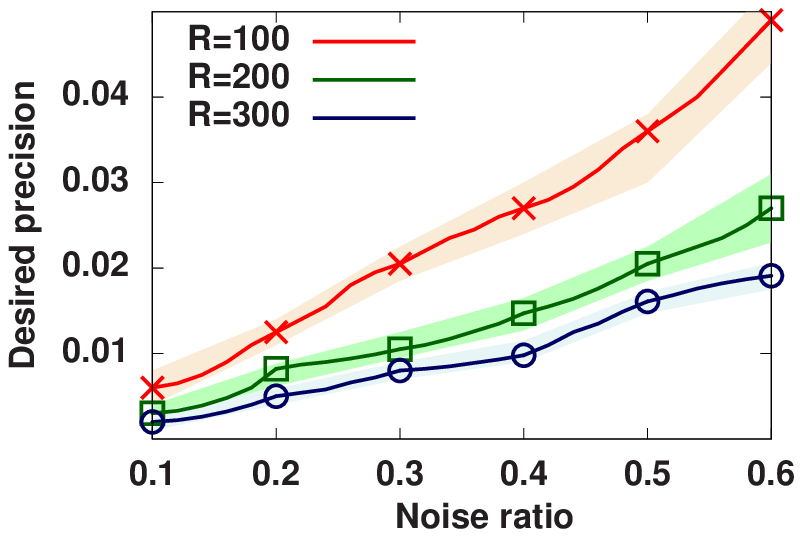} \\
\small{(a) Fixed precision.} & \hspace{0.3cm} \small{(b) Fixed local epoch.}\\
\end{tabular}
\vspace{-0.5cm}
         \caption{Impact of noise ratio on global iterations (communication rounds) with fixed precision and local epochs on LMR dataset.} 
         \label{comm_round1}
         \vspace{-0.5cm}
 \end{figure}

\section{Conclusion and Future work}\label{conc}
In this paper, we proposed a federated learning approach in the presence of noisy labels in the participants' datasets. Unlike the existing work, the Fed-NL trained the model on participants with noisy labels in their datasets. Fed-NL achieved significant performance by exploiting the server's data instances with noise-free labels. We proposed an algorithm for the Fed-NL and performed its convergence analysis, followed by an estimation of the communication rounds for convergence. We also did the evaluations to verify the effectiveness of Fed-NL in the built environment. 

We found the following conclusions from this work: federated learning work successfully only when the negative implications of noisy labels are effectively mitigated; correctly annotated server's dataset can improve participants' performance; the communication round depends upon the purity of the data annotation. Fed-NL technique considers the only noisy labels in the dataset of the participants. We will consider the noise in the input data instances along with labels. We also plan to consider different challenges, such as resource constraints and unbalanced datasets.

\bibliographystyle{ACM-Reference-Format}
\bibliography{refer}


\begin{thebibliography}{29}


\ifx \showCODEN    \undefined \def \showCODEN     #1{\unskip}     \fi
\ifx \showDOI      \undefined \def \showDOI       #1{#1}\fi
\ifx \showISBNx    \undefined \def \showISBNx     #1{\unskip}     \fi
\ifx \showISBNxiii \undefined \def \showISBNxiii  #1{\unskip}     \fi
\ifx \showISSN     \undefined \def \showISSN      #1{\unskip}     \fi
\ifx \showLCCN     \undefined \def \showLCCN      #1{\unskip}     \fi
\ifx \shownote     \undefined \def \shownote      #1{#1}          \fi
\ifx \showarticletitle \undefined \def \showarticletitle #1{#1}   \fi
\ifx \showURL      \undefined \def \showURL       {\relax}        \fi
\providecommand\bibfield[2]{#2}
\providecommand\bibinfo[2]{#2}
\providecommand\natexlab[1]{#1}
\providecommand\showeprint[2][]{arXiv:#2}

\bibitem[\protect\citeauthoryear{Anguita, Ghio, Oneto, Parra~Perez, and
  Reyes~Ortiz}{Anguita et~al\mbox{.}}{2013}]%
        {anguita2013public}
\bibfield{author}{\bibinfo{person}{Davide Anguita}, \bibinfo{person}{Alessandro
  Ghio}, \bibinfo{person}{Luca Oneto}, \bibinfo{person}{Xavier Parra~Perez},
  {and} \bibinfo{person}{Jorge~Luis Reyes~Ortiz}.}
  \bibinfo{year}{2013}\natexlab{}.
\newblock \showarticletitle{A public domain dataset for human activity}. In
  \bibinfo{booktitle}{\emph{Proc. ESANN}}. \bibinfo{pages}{437--442}.
\newblock


\bibitem[\protect\citeauthoryear{Caldas, Duddu, Wu, Li, Kone{\v{c}}n{\`y},
  McMahan, Smith, and Talwalkar}{Caldas et~al\mbox{.}}{2018}]%
        {caldas2018leaf}
\bibfield{author}{\bibinfo{person}{Sebastian Caldas}, \bibinfo{person}{Sai
  Meher~Karthik Duddu}, \bibinfo{person}{Peter Wu}, \bibinfo{person}{Tian Li},
  \bibinfo{person}{Jakub Kone{\v{c}}n{\`y}}, \bibinfo{person}{H~Brendan
  McMahan}, \bibinfo{person}{Virginia Smith}, {and} \bibinfo{person}{Ameet
  Talwalkar}.} \bibinfo{year}{2018}\natexlab{}.
\newblock \showarticletitle{Leaf: A benchmark for federated settings}.
\newblock \bibinfo{journal}{\emph{arXiv preprint arXiv:1812.01097}}
  (\bibinfo{year}{2018}).
\newblock


\bibitem[\protect\citeauthoryear{Chen, Liao, Chen, and Zhang}{Chen
  et~al\mbox{.}}{2019}]%
        {chen19g}
\bibfield{author}{\bibinfo{person}{Pengfei Chen}, \bibinfo{person}{Ben~Ben
  Liao}, \bibinfo{person}{Guangyong Chen}, {and} \bibinfo{person}{Shengyu
  Zhang}.} \bibinfo{year}{2019}\natexlab{}.
\newblock \showarticletitle{Understanding and Utilizing Deep Neural Networks
  Trained with Noisy Labels}. In \bibinfo{booktitle}{\emph{Proc. ICML}}.
  \bibinfo{pages}{1062--1070}.
\newblock


\bibitem[\protect\citeauthoryear{Chen, Yang, Qin, Yu, Chen, and Shen}{Chen
  et~al\mbox{.}}{2020}]%
        {chen2020focus}
\bibfield{author}{\bibinfo{person}{Yiqiang Chen}, \bibinfo{person}{Xiaodong
  Yang}, \bibinfo{person}{Xin Qin}, \bibinfo{person}{Han Yu},
  \bibinfo{person}{Biao Chen}, {and} \bibinfo{person}{Zhiqi Shen}.}
  \bibinfo{year}{2020}\natexlab{}.
\newblock \showarticletitle{Focus: Dealing with label quality disparity in
  federated learning}.
\newblock \bibinfo{journal}{\emph{arXiv preprint arXiv:2001.11359}}
  (\bibinfo{year}{2020}).
\newblock


\bibitem[\protect\citeauthoryear{Cohen, Afshar, Tapson, and Van~Schaik}{Cohen
  et~al\mbox{.}}{2017}]%
        {cohen2017emnist}
\bibfield{author}{\bibinfo{person}{Gregory Cohen}, \bibinfo{person}{Saeed
  Afshar}, \bibinfo{person}{Jonathan Tapson}, {and} \bibinfo{person}{Andre
  Van~Schaik}.} \bibinfo{year}{2017}\natexlab{}.
\newblock \showarticletitle{EMNIST: Extending MNIST to handwritten letters}. In
  \bibinfo{booktitle}{\emph{Proc. IEEE IJCNN}}. \bibinfo{pages}{2921--2926}.
\newblock


\bibitem[\protect\citeauthoryear{Duan, Liu, Chen, Tan, Ren, Qiao, and
  Liang}{Duan et~al\mbox{.}}{2019}]%
        {8988732}
\bibfield{author}{\bibinfo{person}{Moming Duan}, \bibinfo{person}{Duo Liu},
  \bibinfo{person}{Xianzhang Chen}, \bibinfo{person}{Yujuan Tan},
  \bibinfo{person}{Jinting Ren}, \bibinfo{person}{Lei Qiao}, {and}
  \bibinfo{person}{Liang Liang}.} \bibinfo{year}{2019}\natexlab{}.
\newblock \showarticletitle{Astraea: Self-Balancing Federated Learning for
  Improving Classification Accuracy of Mobile Deep Learning Applications}. In
  \bibinfo{booktitle}{\emph{Proc. IEEE ICCD}}. \bibinfo{pages}{246--254}.
\newblock


\bibitem[\protect\citeauthoryear{Gani, Raychoudhury, Edinger, Mokrenko, Cao,
  and Zhang}{Gani et~al\mbox{.}}{2019}]%
        {builtenv}
\bibfield{author}{\bibinfo{person}{Md~Osman Gani}, \bibinfo{person}{Vaskar
  Raychoudhury}, \bibinfo{person}{Janick Edinger}, \bibinfo{person}{Valeria
  Mokrenko}, \bibinfo{person}{Zheng Cao}, {and} \bibinfo{person}{Ce Zhang}.}
  \bibinfo{year}{2019}\natexlab{}.
\newblock \showarticletitle{Smart Surface Classification for Accessible Routing
  through Built Environment: A Crowd-Sourced Approach}. In
  \bibinfo{booktitle}{\emph{Proc. ACM BuildSys}}. \bibinfo{pages}{11–20}.
\newblock


\bibitem[\protect\citeauthoryear{Gao, Wang, Liu, Billah, and Campbell}{Gao
  et~al\mbox{.}}{2021}]%
        {gao2021decentralized}
\bibfield{author}{\bibinfo{person}{Jiechao Gao}, \bibinfo{person}{Wenpeng
  Wang}, \bibinfo{person}{Zetian Liu}, \bibinfo{person}{Md~Fazlay Rabbi~Masum
  Billah}, {and} \bibinfo{person}{Bradford Campbell}.}
  \bibinfo{year}{2021}\natexlab{}.
\newblock \showarticletitle{Decentralized federated learning framework for the
  neighborhood: a case study on residential building load forecasting}. In
  \bibinfo{booktitle}{\emph{Proc. SenSys}}. \bibinfo{pages}{453--459}.
\newblock


\bibitem[\protect\citeauthoryear{Han and Zhang}{Han and Zhang}{2020}]%
        {Han2020}
\bibfield{author}{\bibinfo{person}{Yufei Han} {and} \bibinfo{person}{Xiangliang
  Zhang}.} \bibinfo{year}{2020}\natexlab{}.
\newblock \showarticletitle{Robust Federated Learning via Collaborative Machine
  Teaching}. In \bibinfo{booktitle}{\emph{Proc. AAAI}}.
  \bibinfo{pages}{4075--4082}.
\newblock


\bibitem[\protect\citeauthoryear{Hendrycks, Mazeika, Wilson, and
  Gimpel}{Hendrycks et~al\mbox{.}}{2018}]%
        {hendrycks2018}
\bibfield{author}{\bibinfo{person}{Dan Hendrycks}, \bibinfo{person}{Mantas
  Mazeika}, \bibinfo{person}{Duncan Wilson}, {and} \bibinfo{person}{Kevin
  Gimpel}.} \bibinfo{year}{2018}\natexlab{}.
\newblock \showarticletitle{Using trusted data to train deep networks on labels
  corrupted by severe noise}. In \bibinfo{booktitle}{\emph{Proc. NIPS}}.
  \bibinfo{pages}{10456--10465}.
\newblock


\bibitem[\protect\citeauthoryear{Krizhevsky, Hinton, et~al\mbox{.}}{Krizhevsky
  et~al\mbox{.}}{2009}]%
        {krizhevsky2009learning}
\bibfield{author}{\bibinfo{person}{Alex Krizhevsky}, \bibinfo{person}{Geoffrey
  Hinton}, {et~al\mbox{.}}} \bibinfo{year}{2009}\natexlab{}.
\newblock \showarticletitle{Learning multiple layers of features from tiny
  images}.
\newblock  (\bibinfo{year}{2009}).
\newblock


\bibitem[\protect\citeauthoryear{Li, Fu, Han, Xu, and Shao}{Li
  et~al\mbox{.}}{2021}]%
        {li2021federated}
\bibfield{author}{\bibinfo{person}{Li Li}, \bibinfo{person}{Huazhu Fu},
  \bibinfo{person}{Bo Han}, \bibinfo{person}{Cheng-Zhong Xu}, {and}
  \bibinfo{person}{Ling Shao}.} \bibinfo{year}{2021}\natexlab{}.
\newblock \showarticletitle{Federated Noisy Client Learning}.
\newblock \bibinfo{journal}{\emph{arXiv preprint arXiv:2106.13239}}
  (\bibinfo{year}{2021}).
\newblock


\bibitem[\protect\citeauthoryear{Li, Sahu, Zaheer, Sanjabi, Talwalkar, and
  Smith}{Li et~al\mbox{.}}{2020}]%
        {li2020federated}
\bibfield{author}{\bibinfo{person}{Tian Li}, \bibinfo{person}{Anit~Kumar Sahu},
  \bibinfo{person}{Manzil Zaheer}, \bibinfo{person}{Maziar Sanjabi},
  \bibinfo{person}{Ameet Talwalkar}, {and} \bibinfo{person}{Virginia Smith}.}
  \bibinfo{year}{2020}\natexlab{}.
\newblock \showarticletitle{Federated optimization in heterogeneous networks}.
\newblock \bibinfo{journal}{\emph{Proc. MLSys}}  \bibinfo{volume}{2}
  (\bibinfo{year}{2020}), \bibinfo{pages}{429--450}.
\newblock


\bibitem[\protect\citeauthoryear{Li, Huang, Yang, Wang, and Zhang}{Li
  et~al\mbox{.}}{2019}]%
        {li2019convergence}
\bibfield{author}{\bibinfo{person}{Xiang Li}, \bibinfo{person}{Kaixuan Huang},
  \bibinfo{person}{Wenhao Yang}, \bibinfo{person}{Shusen Wang}, {and}
  \bibinfo{person}{Zhihua Zhang}.} \bibinfo{year}{2019}\natexlab{}.
\newblock \showarticletitle{On the Convergence of FedAvg on Non-IID Data}. In
  \bibinfo{booktitle}{\emph{Proc. ICLR}}. \bibinfo{pages}{1--26}.
\newblock


\bibitem[\protect\citeauthoryear{Liu and Tao}{Liu and Tao}{2016}]%
        {7159100}
\bibfield{author}{\bibinfo{person}{Tongliang Liu} {and}
  \bibinfo{person}{Dacheng Tao}.} \bibinfo{year}{2016}\natexlab{}.
\newblock \showarticletitle{Classification with Noisy Labels by Importance
  Reweighting}.
\newblock \bibinfo{journal}{\emph{IEEE Transactions on Pattern Analysis and
  Machine Intelligence}} \bibinfo{volume}{38}, \bibinfo{number}{3}
  (\bibinfo{year}{2016}), \bibinfo{pages}{447--461}.
\newblock


\bibitem[\protect\citeauthoryear{Luo, Li, Wang, Huang, and Tassiulas}{Luo
  et~al\mbox{.}}{2021}]%
        {luo2021cost}
\bibfield{author}{\bibinfo{person}{Bing Luo}, \bibinfo{person}{Xiang Li},
  \bibinfo{person}{Shiqiang Wang}, \bibinfo{person}{Jianwei Huang}, {and}
  \bibinfo{person}{Leandros Tassiulas}.} \bibinfo{year}{2021}\natexlab{}.
\newblock \showarticletitle{Cost-effective federated learning design}. In
  \bibinfo{booktitle}{\emph{Proc. IEEE INFOCOM}}. \bibinfo{pages}{1--10}.
\newblock


\bibitem[\protect\citeauthoryear{McMahan, Moore, Ramage, Hampson, and
  y~Arcas}{McMahan et~al\mbox{.}}{2017}]%
        {mcmahan2017}
\bibfield{author}{\bibinfo{person}{Brendan McMahan}, \bibinfo{person}{Eider
  Moore}, \bibinfo{person}{Daniel Ramage}, \bibinfo{person}{Seth Hampson},
  {and} \bibinfo{person}{Blaise~Aguera y Arcas}.}
  \bibinfo{year}{2017}\natexlab{}.
\newblock \showarticletitle{Communication-efficient learning of deep networks
  from decentralized data}. In \bibinfo{booktitle}{\emph{Proc. AISTATS}}.
  \bibinfo{pages}{1273--1282}.
\newblock


\bibitem[\protect\citeauthoryear{{Mishra}, {Gupta}, {Gupta}, and
  {Dutta}}{{Mishra} et~al\mbox{.}}{2020}]%
        {9164991}
\bibfield{author}{\bibinfo{person}{R. {Mishra}}, \bibinfo{person}{A. {Gupta}},
  \bibinfo{person}{H.~P. {Gupta}}, {and} \bibinfo{person}{T. {Dutta}}.}
  \bibinfo{year}{2020}\natexlab{}.
\newblock \showarticletitle{A Sensors based Deep Learning Model for Unseen
  Locomotion Mode Identification using Multiple Semantic Matrices}.
\newblock \bibinfo{journal}{\emph{IEEE Transactions on Mobile Computing}}
  (\bibinfo{year}{2020}), \bibinfo{pages}{1--1}.
\newblock
\newblock
\shownote{doi: \url{10.1109/TMC.2020.3015546}.}


\bibitem[\protect\citeauthoryear{Sater and Hamza}{Sater and Hamza}{2021}]%
        {sater2021federated}
\bibfield{author}{\bibinfo{person}{Raed~Abdel Sater} {and}
  \bibinfo{person}{A~Ben Hamza}.} \bibinfo{year}{2021}\natexlab{}.
\newblock \showarticletitle{A federated learning approach to anomaly detection
  in smart buildings}.
\newblock \bibinfo{journal}{\emph{ACM Transactions on Internet of Things}}
  \bibinfo{volume}{2}, \bibinfo{number}{4} (\bibinfo{year}{2021}),
  \bibinfo{pages}{1--23}.
\newblock


\bibitem[\protect\citeauthoryear{Sattler, Korjakow, Rischke, and Samek}{Sattler
  et~al\mbox{.}}{2021}]%
        {9632275}
\bibfield{author}{\bibinfo{person}{Felix Sattler}, \bibinfo{person}{Tim
  Korjakow}, \bibinfo{person}{Roman Rischke}, {and} \bibinfo{person}{Wojciech
  Samek}.} \bibinfo{year}{2021}\natexlab{}.
\newblock \showarticletitle{FEDAUX: Leveraging Unlabeled Auxiliary Data in
  Federated Learning}.
\newblock \bibinfo{journal}{\emph{IEEE Transactions on Neural Networks and
  Learning Systems}} (\bibinfo{year}{2021}), \bibinfo{pages}{1--13}.
\newblock
\newblock
\shownote{doi: \url{10.1109/TNNLS.2021.3129371}.}


\bibitem[\protect\citeauthoryear{{SHL Challenge}}{{SHL Challenge}}{2022}]%
        {shl2}
\bibfield{author}{\bibinfo{person}{{SHL Challenge}}.}
  \bibinfo{year}{2022}\natexlab{}.
\newblock
\newblock
\urldef\tempurl%
\url{http://www.shl-dataset.org/activity-recognition-challenge/}
\showURL{%
\tempurl}


\bibitem[\protect\citeauthoryear{Tuor, Wang, Ko, Liu, and Leung}{Tuor
  et~al\mbox{.}}{2021}]%
        {9412599}
\bibfield{author}{\bibinfo{person}{Tiffany Tuor}, \bibinfo{person}{Shiqiang
  Wang}, \bibinfo{person}{Bong~Jun Ko}, \bibinfo{person}{Changchang Liu}, {and}
  \bibinfo{person}{Kin~K. Leung}.} \bibinfo{year}{2021}\natexlab{}.
\newblock \showarticletitle{Overcoming Noisy and Irrelevant Data in Federated
  Learning}. In \bibinfo{booktitle}{\emph{Proc. ICPR}}.
  \bibinfo{pages}{5020--5027}.
\newblock


\bibitem[\protect\citeauthoryear{Wang, Liu, Liang, Joshi, and Poor}{Wang
  et~al\mbox{.}}{2020}]%
        {tackling}
\bibfield{author}{\bibinfo{person}{Jianyu Wang}, \bibinfo{person}{Qinghua Liu},
  \bibinfo{person}{Hao Liang}, \bibinfo{person}{Gauri Joshi}, {and}
  \bibinfo{person}{H.~Vincent Poor}.} \bibinfo{year}{2020}\natexlab{}.
\newblock \showarticletitle{Tackling the Objective Inconsistency Problem in
  Heterogeneous Federated Optimization}. In \bibinfo{booktitle}{\emph{Proc.
  NIPS}}. \bibinfo{pages}{7611--7623}.
\newblock


\bibitem[\protect\citeauthoryear{Wang, Xu, Wang, and Zhu}{Wang
  et~al\mbox{.}}{2021}]%
        {wang2021addressing}
\bibfield{author}{\bibinfo{person}{Lixu Wang}, \bibinfo{person}{Shichao Xu},
  \bibinfo{person}{Xiao Wang}, {and} \bibinfo{person}{Qi Zhu}.}
  \bibinfo{year}{2021}\natexlab{}.
\newblock \showarticletitle{Addressing class imbalance in federated learning}.
  In \bibinfo{booktitle}{\emph{Proc. AAAI}}, Vol.~\bibinfo{volume}{35}.
  \bibinfo{pages}{10165--10173}.
\newblock


\bibitem[\protect\citeauthoryear{Wang, Tuor, Salonidis, Leung, Makaya, He, and
  Chan}{Wang et~al\mbox{.}}{2019}]%
        {wang2019adaptive}
\bibfield{author}{\bibinfo{person}{Shiqiang Wang}, \bibinfo{person}{Tiffany
  Tuor}, \bibinfo{person}{Theodoros Salonidis}, \bibinfo{person}{Kin~K Leung},
  \bibinfo{person}{Christian Makaya}, \bibinfo{person}{Ting He}, {and}
  \bibinfo{person}{Kevin Chan}.} \bibinfo{year}{2019}\natexlab{}.
\newblock \showarticletitle{Adaptive federated learning in resource constrained
  edge computing systems}.
\newblock \bibinfo{journal}{\emph{IEEE Journal on Selected Areas in
  Communications}} \bibinfo{volume}{37}, \bibinfo{number}{6}
  (\bibinfo{year}{2019}), \bibinfo{pages}{1205--1221}.
\newblock


\bibitem[\protect\citeauthoryear{Xue, Niu, Zheng, Tang, Lyu, Wu, and Chen}{Xue
  et~al\mbox{.}}{2021}]%
        {xue2021toward}
\bibfield{author}{\bibinfo{person}{Yihao Xue}, \bibinfo{person}{Chaoyue Niu},
  \bibinfo{person}{Zhenzhe Zheng}, \bibinfo{person}{Shaojie Tang},
  \bibinfo{person}{Chengfei Lyu}, \bibinfo{person}{Fan Wu}, {and}
  \bibinfo{person}{Guihai Chen}.} \bibinfo{year}{2021}\natexlab{}.
\newblock \showarticletitle{Toward Understanding the Influence of Individual
  Clients in Federated Learning}. In \bibinfo{booktitle}{\emph{Proc. AAAI}}.
  \bibinfo{pages}{10560--10567}.
\newblock


\bibitem[\protect\citeauthoryear{Yang, Fang, and Liu}{Yang
  et~al\mbox{.}}{2021}]%
        {yang2021achieving}
\bibfield{author}{\bibinfo{person}{Haibo Yang}, \bibinfo{person}{Minghong
  Fang}, {and} \bibinfo{person}{Jia Liu}.} \bibinfo{year}{2021}\natexlab{}.
\newblock \showarticletitle{Achieving Linear Speedup with Partial Worker
  Participation in Non-IID Federated Learning}.
\newblock \bibinfo{journal}{\emph{Proc. ICLR}} (\bibinfo{year}{2021}),
  \bibinfo{pages}{1--23}.
\newblock


\bibitem[\protect\citeauthoryear{Yang, Park, Byun, and Kim}{Yang
  et~al\mbox{.}}{2022}]%
        {yang2020robust}
\bibfield{author}{\bibinfo{person}{Seunghan Yang}, \bibinfo{person}{Hyoungseob
  Park}, \bibinfo{person}{Junyoung Byun}, {and} \bibinfo{person}{Changick
  Kim}.} \bibinfo{year}{2022}\natexlab{}.
\newblock \showarticletitle{Robust Federated Learning with Noisy Labels}.
\newblock \bibinfo{journal}{\emph{IEEE Intelligent Systems}}
  (\bibinfo{year}{2022}), \bibinfo{pages}{1--10}.
\newblock
\newblock
\shownote{doi: \url{10.1109/MIS.2022.3151466}.}


\bibitem[\protect\citeauthoryear{Yu, Liu, Gong, Batmanghelich, and Tao}{Yu
  et~al\mbox{.}}{2018}]%
        {yu2018efficient}
\bibfield{author}{\bibinfo{person}{Xiyu Yu}, \bibinfo{person}{Tongliang Liu},
  \bibinfo{person}{Mingming Gong}, \bibinfo{person}{Kayhan Batmanghelich},
  {and} \bibinfo{person}{Dacheng Tao}.} \bibinfo{year}{2018}\natexlab{}.
\newblock \showarticletitle{An efficient and provable approach for mixture
  proportion estimation using linear independence assumption}. In
  \bibinfo{booktitle}{\emph{Proc. CVPR}}. \bibinfo{pages}{4480--4489}.
\newblock


\end{thebibliography}
\end{document}